\PassOptionsToPackage{unicode}{hyperref}
\PassOptionsToPackage{hyphens}{url}
\PassOptionsToPackage{dvipsnames,svgnames,x11names}{xcolor}
\documentclass[
  12pt]{article}

\usepackage{amsmath,amssymb}
\usepackage{iftex}
\ifPDFTeX
  \usepackage[T1]{fontenc}
  \usepackage[utf8]{inputenc}
  \usepackage{textcomp} 
\else 
  \usepackage{unicode-math}
  \defaultfontfeatures{Scale=MatchLowercase}
  \defaultfontfeatures[\rmfamily]{Ligatures=TeX,Scale=1}
\fi
\usepackage{lmodern}
\ifPDFTeX\else  
\fi
\IfFileExists{upquote.sty}{\usepackage{upquote}}{}
\IfFileExists{microtype.sty}{
  \usepackage[]{microtype}
  \UseMicrotypeSet[protrusion]{basicmath} 
}{}
\makeatletter
\@ifundefined{KOMAClassName}{
  \IfFileExists{parskip.sty}{%
    \usepackage{parskip}
  }{
    \setlength{\parindent}{0pt}
    \setlength{\parskip}{6pt plus 2pt minus 1pt}}
}{
  \KOMAoptions{parskip=half}}
\makeatother
\usepackage{xcolor}
\makeatletter
\ifx\paragraph\undefined\else
  \let\oldparagraph\paragraph
  \renewcommand{\paragraph}{
    \@ifstar
      \xxxParagraphStar
      \xxxParagraphNoStar
  }
  \newcommand{\xxxParagraphStar}[1]{\oldparagraph*{#1}\mbox{}}
  \newcommand{\xxxParagraphNoStar}[1]{\oldparagraph{#1}\mbox{}}
\fi
\ifx\subparagraph\undefined\else
  \let\oldsubparagraph\subparagraph
  \renewcommand{\subparagraph}{
    \@ifstar
      \xxxSubParagraphStar
      \xxxSubParagraphNoStar
  }
  \newcommand{\xxxSubParagraphStar}[1]{\oldsubparagraph*{#1}\mbox{}}
  \newcommand{\xxxSubParagraphNoStar}[1]{\oldsubparagraph{#1}\mbox{}}
\fi
\makeatother

\usepackage{longtable,booktabs,array}
\usepackage{calc} 
\usepackage{etoolbox}
\makeatletter
\patchcmd\longtable{\par}{\if@noskipsec\mbox{}\fi\par}{}{}
\makeatother
\IfFileExists{footnotehyper.sty}{\usepackage{footnotehyper}}{\usepackage{footnote}}
\makesavenoteenv{longtable}
\usepackage{graphicx}
\makeatletter
\def\maxwidth{\ifdim\Gin@nat@width>\linewidth\linewidth\else\Gin@nat@width\fi}
\def\maxheight{\ifdim\Gin@nat@height>\textheight\textheight\else\Gin@nat@height\fi}
\makeatother
\setkeys{Gin}{width=\maxwidth,height=\maxheight,keepaspectratio}
\makeatletter
\def\fps@figure{htbp}
\makeatother

\makeatletter
\@ifpackageloaded{caption}{}{\usepackage{caption}}
\AtBeginDocument{%
\ifdefined\contentsname
  \renewcommand*\contentsname{Table of contents}
\else
  \newcommand\contentsname{Table of contents}
\fi
\ifdefined\listfigurename
  \renewcommand*\listfigurename{List of Figures}
\else
  \newcommand\listfigurename{List of Figures}
\fi
\ifdefined\listtablename
  \renewcommand*\listtablename{List of Tables}
\else
  \newcommand\listtablename{List of Tables}
\fi
\ifdefined\figurename
  \renewcommand*\figurename{Figure}
\else
  \newcommand\figurename{Figure}
\fi
\ifdefined\tablename
  \renewcommand*\tablename{Table}
\else
  \newcommand\tablename{Table}
\fi
}
\@ifpackageloaded{float}{}{\usepackage{float}}
\floatstyle{ruled}
\@ifundefined{c@chapter}{\newfloat{codelisting}{h}{lop}}{\newfloat{codelisting}{h}{lop}[chapter]}
\floatname{codelisting}{Listing}

\makeatother
\makeatletter
\@ifpackageloaded{caption}{}{\usepackage{caption}}
\@ifpackageloaded{subcaption}{}{\usepackage{subcaption}}
\makeatother

\ifLuaTeX
  \usepackage{selnolig}  
\fi
\usepackage[]{natbib}
\usepackage{bookmark}

\IfFileExists{xurl.sty}{\usepackage{xurl}}{} 
\hypersetup{
  pdftitle={Title},
  pdfauthor={Author 1; Author 2},
  pdfkeywords={3 to 6 keywords, that do not appear in the title},
  colorlinks=true,
  linkcolor={blue},
  filecolor={Maroon},
  citecolor={Blue},
  urlcolor={Blue},
  pdfcreator={LaTeX via pandoc}}

\newcommand{\anon}{1}

\usepackage{algorithm}      
\usepackage{algorithmic}  
\usepackage{enumitem}
\usepackage{amsmath, amssymb, amsthm}
\newtheorem{theorem}{Theorem}
\newtheorem{proposition}{Proposition}

\usepackage{tikz}
\usetikzlibrary{arrows.meta,positioning,calc,fit,decorations.pathreplacing}
\usepackage[font=footnotesize]{caption} 

\begin{document}

\def\spacingset#1{\renewcommand{\baselinestretch}%
{#1}\small\normalsize} \spacingset{1}


\if1\anon
{
  \title{\bf High-Dimensional Statistical Process Control via Manifold Fitting and Learning}
  \author{Ismail Burak Tas\hspace{.2cm}\\
    and \\
    Enrique del Castillo\footnote{Corresponding author. Dr. Castillo is Distinguished Professor of Industrial \& Manufacturing Engineering and Professor of Statistics. e-mail: exd13@psu.edu} \\
    {\small Engineering Statistics and Machine Learning Laboratory}\\
{\small Department of Industrial and Manufacturing Engineering and Dept. of Statistics}\\
{\small The Pennsylvania State University, University Park, PA 16802, USA}}
  \maketitle
} \fi

\if0\anon
{
  \bigskip
  \bigskip
  \bigskip
  \begin{center}
    {\LARGE\bf Title}
\end{center}
  \medskip
} \fi

\bigskip
\begin{abstract}
We address the Statistical Process Control (SPC) of high-dimensional, dynamic industrial processes from two complementary perspectives: manifold fitting and manifold learning, both of which assume data lies on an underlying nonlinear, lower dimensional space. We propose two distinct monitoring frameworks for online or 'phase II' Statistical Process Control (SPC). The first method leverages state-of-the-art techniques in manifold fitting to accurately approximate the manifold where the data resides within the ambient high-dimensional space. It then monitors deviations from this manifold using a novel scalar distribution-free control chart. In contrast, the second method adopts a more traditional approach, akin to those used in linear dimensionality reduction SPC techniques, by first embedding the data into a lower-dimensional space before monitoring the embedded observations. We prove how both methods provide a controllable Type I error probability, after which they are contrasted for their corresponding fault detection ability. Extensive numerical experiments on a synthetic process and on a replicated Tennessee Eastman Process show that the conceptually simpler manifold-fitting approach achieves performance competitive with, and sometimes superior to, the more classical lower-dimensional manifold monitoring methods. In addition, we demonstrate the practical applicability of the proposed manifold-fitting approach by successfully detecting surface anomalies in a real image dataset of electrical commutators.
\end{abstract}

\noindent%
\vfill

\newpage
\spacingset{1.6} 

\section{Introduction}\label{sec-intro}

Statistical Process Control (SPC) is a core methodology in quality control and industrial statistics, designed to monitor the stability of complex systems over time and ultimately to reduce their variability by identifying assignable causes of variation that will lead to their removal. In this paper we study the so-called ``Phase II'' SPC or online monitoring for high dimensional processes. 
In modern data-rich environments, observations are often high-dimensional, which significantly reduces the effectiveness of traditional multivariate control charts. For instance, the performance of the well-known Hotelling’s $T^2$ control chart \citep{hotelling1947multivariate} degrades as the dimensionality $D$ increases~\citep{jiang2008theoretical}, and it becomes inapplicable in the traditional sense when the number of Phase I observations $m$ is less than $D$.

A widely adopted strategy to deal with high-dimensional data is to assume data lie on a lower-dimensional manifold. This assumption enables dimensionality reduction, and multivariate SPC studies have pursued this direction, most notably through Principal Component Analysis (PCA) (e.g.,~\cite{jackson1991user},~\cite{macgregor1995statistical}). However, PCA-based monitoring methods rely on the assumption that the underlying manifold $\mathcal{M}$ is linear \citep{ma2012manifold}, that is, they are effective only when $\mathcal{M}$ can be decomposed along linear directions of variance. In reality, many dynamic processes, particularly those governed by feedback loops, do not generate data that lie on a linear manifold. In this paper, we propose two new and complementary methods for Phase II SPC of high-dimensional dynamic processes: an innovative manifold fitting approach, and what could be thought to be a more traditional approach, that of manifold learning, both of which can handle an underlying nonlinear data manifold.

Manifold learning techniques (e.g., ISOMAP \citep{tenenbaum2000global}, Diffusion Maps \citep{coifman2006diffusion}, Tangent Space Alingment \citep{zhang2004principal}, Locally Linear Embedding \citep{roweis2000nonlinear}, and Laplacian Eigenmaps \citep{belkin2003laplacian}) assume that the high-dimensional observations lie near or on $\mathcal{M}\subset \mathbb{R}^{D}$ and seek to approximate an embedding function $\hat{f}:\mathcal{M}\to \mathbb{R}^d$, where $d\ll D$. Despite their success in various machine learning applications, these nonlinear dimensionality reduction methods have seen limited adoption in the SPC literature. The primary reason is their lack of a natural and efficient out-of-sample extension. Specifically, these methods do not provide an explicit mapping $\hat{f}$, instead, they directly produce the lower-dimensional representations of the training data. This limitation presents a significant challenge for real-time process monitoring, where new observations must be quickly and consistently projected into the low-dimensional space. However, if the mapping $\hat{f}$ is constrained to be linear, it can be recovered and subsequently used to embed new observations. Such approaches are known as linear approximations of nonlinear manifold learning methods. Examples include Locality Preserving Projections (LPP) \citep{lpp}, Neighborhood Preserving Embedding (NPE) \citep{npe}, and their variants, which have been recently adopted for process monitoring by Chemical Engineers (e.g., \cite{li2015}, \cite{zhan2019improved}, \cite{cui2022nonparametric}, \cite{ma2015fault}). A common approach in these works is to construct two control charts, one for the low-dimensional space and another for the reconstruction errors in the ambient space, which requires consideration of the joint behavior of two possibly dependent charts.

More recently, related research has emerged under the name of ``manifold fitting''. The purpose of manifold fitting is to reconstruct a smooth manifold $\hat{\mathcal{M}} \subset \mathbb{R}^D$ that approximates the geometry and topology of $\mathcal{M}$, based solely on the observed high-dimensional data. \cite{Sober2020} and \cite{aizenbud2021non} first approximate the tangent space of $\mathcal{M}$ on a particular point and then fit a polynomial to locally approximate $\mathcal{M}$ as a function defined on the approximated tangent vector space. \cite{fefferman2023fitting} constructs a smooth manifold by projecting noisy data onto locally estimated discs and refining these projections through geometry-aware averaging to obtain $\hat{\mathcal{M}}$. \cite{yao2023manifold} proposed a relatively simple approach that estimates the contraction direction by constructing a Euclidean ball centered at each noisy point and then projecting each point by moving it along the estimated direction. The same authors later introduced a CycleGAN-based manifold fitting method in \cite{yao2024manifold}, integrating the algorithm in \cite{yao2023manifold}. Both approaches assume a known underlying noise level, which in practical applications is difficult to justify. The goal of manifold fitting is to recover a manifold $\hat{\mathcal{M}}$ directly in the ambient space, in contrast to manifold learning methods, which seek to embed $\mathcal{M}$ into a lower-dimensional space $\mathbb{R}^{d}$. Figure \ref{fitting_embedding} below illustrates this distinction.
 \begin{figure}[h!]
 \centering
  \includegraphics[width=0.6\textwidth]{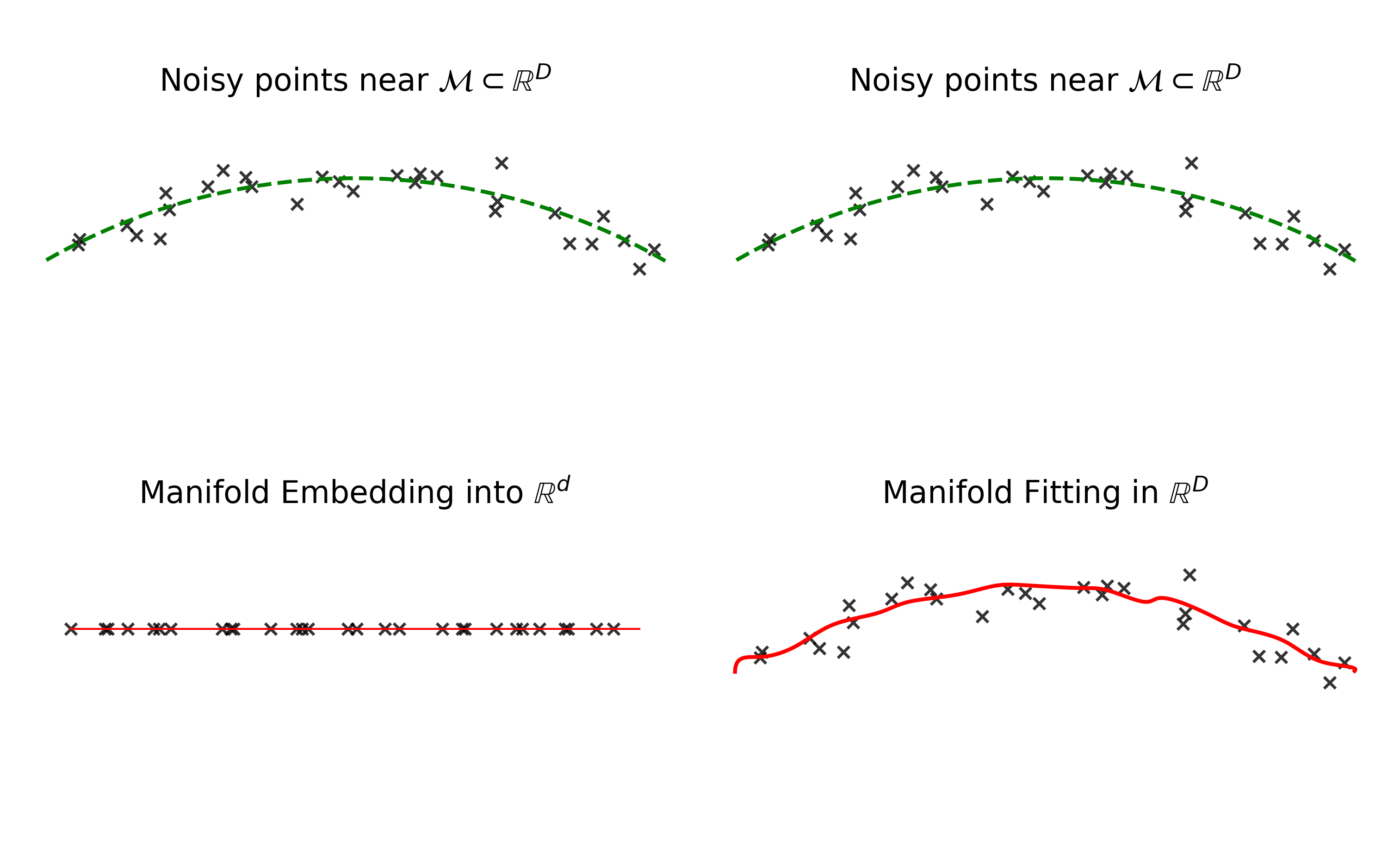}
 \caption{Manifold embedding and fitting, adapted from \cite{yao2023manifold}}
  \label{fitting_embedding}
 \end{figure}

We consider a process that generates high-dimensional observations $\{Y_t\}_{t=-m+1}^n \subset \mathbb{R}^D$, where each observation $Y_t$ is a noisy measurement of an unobserved latent state $X_t \in \mathbb{R}^D$. Specifically,
$$Y_t = X_t + \mathcal{E}_t, \quad \text{for} \quad t=-m+1,\dots,0,1,\dots,n,$$
where $\mathcal{E}_t \overset{\text{i.i.d.}}{\sim} \mathcal{N}(0, \sigma^2 I_D)$ denotes a stochastic noise term independent of $X_t$. We assume the latent process $\{X_t\}$ evolves according to a stationary and ergodic stochastic process, generated by a possibly nonlinear and unknown stochastic evolution rule $F: \mathcal{M} \to \mathcal{M}$, where $\mathcal{M} \subset \mathbb{R}^D$ is a compact, at least twice-differentiable manifold of intrinsic dimension $d < D$. The stationary distribution of $\{X_t\}$ is assumed to be uniform with respect to the $d$-dimensional Hausdorff measure on $\mathcal{M}$.
Our primary interest lies in detecting a sustained \emph{mean shift} in the ambient process $\{Y_t\}$ at an unknown change-point $\tau > 0$. After $\tau$, we assume that the mean of $Y_t$ shifts to a new value, i.e., $Y_t+\Delta$, where $\Delta \in \mathbb{R}^D$ is a non-zero vector. Detection of other types of faults will be considered in Section~\ref{performance_analysis}.

Most of the classical SPC tools are developed under an independent identically distributed (i.i.d.) data-generating process assumption. Common approaches in SPC to handle autocorrelation data created by the stochastic evolution function $F$ involve preprocessing the data to remove the autocorrelation, typically using techniques such as decorrelation or time-series modeling. 
After preprocessing, the transformed data approximate i.i.d. behavior, enabling the use of traditional SPC methods (e.g., \cite{montgomery1991some}, \cite{alwan1988time}, \cite{lu1999control}). Recent instances of this general approach are \cite{qiu2022transparent}, who proposed a sequential learning procedure that applies serial decorrelation before a conventional control chart, or \cite{qiu2020new} who developed a control chart incorporating serial decorrelation, and \cite{xie2024general} who introduced a general framework for monitoring correlated processes via sequential parameter updates to decorrelate observations.

\subsection{Overview of a novel Manifold Fitting method for SPC and paper organization
}
Provided that the process is in a state of statistical control and $\sigma$ is sufficiently small, ${Y_t}$ lies near $\mathcal{M}$. This motivates us to monitor deviations from $\mathcal{M}$ through the distance 
$$\text{dist}(Y_t,\mathcal{M})=\inf_{x\in\mathcal{M}}||Y_t-x||_2$$
without modeling the mapping $F$ contrary to the traditional SPC approaches. The distribution of $\text{dist}(Y_t,\mathcal{M})$ depends on $\mathcal{M}$ and is influenced by its geometric properties (e.g., intrinsic dimension $d$, curvature). If a reliable fit $\hat{\mathcal{M}}$ of $\mathcal{M}$ is obtained from Phase I data, on-line monitoring can then proceed by tracking $\text{dist}(Y_t,\hat{\mathcal{M}})$.

Remarkably, this purely geometric approach can be used to monitor complex processes characterized by a wide variety of signal types or streaming data, where traditional SPC methods often struggle due to the curse of dimensionality or model misspecification. By exploiting data geometry, this approach can be used in on-line monitoring of multiple correlated sensors, functional data or multivariate time series and can be extended to high-dimensional non-process data (e.g., images). In all cases, reducing monitoring to the scalar statistic $\text{dist}(Y_t,\hat{\mathcal{M}})$ allows the use of a univariate control chart, without requiring explicit models or distributional assumptions.

The proposed method offers a controllable Type I error (or equivalently, in-control average run length, ARL$_{in}$), and the extensive numerical evidence we provide demonstrate that it performs competitively with manifold learning-based SPC approaches in a variety of out-of-control scenarios. Unlike prior work, our method requires neither dimensionality reduction nor multiple control charts. Moreover, because prior approaches based on manifold learning do not account for serial correlations or guarantee a controllable ARL$_{in}$, we fill this gap by filtering temporal dependencies within the learned low-dimensional representation, thereby enabling both practical implementation and theoretical validation of prior approaches.

The paper is organized as follows. Section~\ref{sec:previous_work} reviews the literature on manifold learning methods in Phase II SPC. Section~\ref{manifold_fitting} presents our proposed framework, which fits $\mathcal{M}$ and implements a univariate control chart. Section~\ref{manifold_learning} introduces a manifold learning–based SPC framework with a controllable ARL$_\text{in}$. Section~\ref{performance_analysis} presents a performance comparison of both frameworks on a synthetic process and on a high-dimensional chemical process generated by replicating the Tennessee Eastman Process simulator. It also demonstrates the effectiveness of the proposed framework for surface defect detection on a real image dataset of electrical commutators.

\section{Previous Work on Manifold-based SPC}\label{sec:previous_work}

The curse of dimensionality significantly reduces the fault detection power of traditional multivariate SPC methods. For instance, the detection power of Hotelling’s $T^2$ control chart, even when constructed using the true mean $\mu$ and covariance matrix $\Sigma$, declines as the dimensionality $D$ increases, given a fixed Mahalanobis distance of the mean shift $\Delta$, i.e., $\Delta^\top \Sigma^{-1} \Delta$ \citep{jiang2008theoretical}. A fundamental strategy to address this issue is to embed high-dimensional observations into a lower-dimensional space. In differential geometry, an embedding is defined as a smooth map $f: \mathcal{M} \to \mathcal{N}$ between two $d$-dimensional manifolds $\mathcal{M}$ and $\mathcal{N}$, such that the inverse map $f^{-1}: f(\mathcal{M}) \subset\mathcal{N}  \to \mathcal{M}$ is also smooth. According to Whitney's Embedding Theorem, any $d$-dimensional smooth manifold $\mathcal{M}$ can be embedded in $\mathbb{R}^{d_0}$, where $d \leq d_0 \leq 2d$ \citep{lee2012introduction}. Therefore, if a valid embedding function $f$ can be found, any observation in $\mathbb{R}^D$ can be represented in $O(d)$ dimensions, leading to a substantial simplification when $d \ll D$.

Researchers have widely adopted the strategy of reducing the dimensionality and monitoring the lower-dimensional space for Phase II SPC. The general procedure involves estimating an approximation $\hat{f}$ of the embedding function $f$ using Phase I observations, then applying $\hat{f}$ to embed Phase II observations into the lower-dimensional space, where monitoring is performed. A limitation of dimensionality reduction is that it may fail to capture fault directions lying outside the retained subspace, so certain shifts cannot be detected in the lower-dimensional space. To address this, the monitoring strategy is generally complemented by an additional control chart built in the ambient space $\mathbb{R}^D$, which monitors the reconstruction error of each observation. This yields a two-chart monitoring scheme (e.g., $T^2$ chart for the lower-dimensional space complemented by a squared prediction error (SPE) chart for the reconstruction errors) in which both charts contribute to the overall in-control performance, and whose joint ARL$_{in}$ should therefore be properly adjusted. However, explicit consideration of this adjustment appears to be absent in the literature (e.g., \cite{macgregor1995statistical}, \cite{ma2015fault}, \cite{cui2022nonparametric}). Even when the joint ARL$_{in}$ is considered, this two-chart monitoring scheme is effectively equivalent to monitoring the entire ambient space, with certain dimensions prioritized due to the manifold learning methods. In our approach, this issue is avoided by fitting $\mathcal{M}$ directly in the ambient space, as discussed in Section \ref{sec:toy_example}.


The assumption of a linear manifold in SPC methods has been recently relaxed in the Chemometrics literature (e.g., \cite{li2015}, \cite{zhan2019improved}, \cite{cui2022nonparametric}, \cite{ma2015fault}). In these studies, dimensionality reduction techniques such as NPE, LPP along with their variants, have been successfully employed to learn a linear embedding function $\hat{f}$ for non-linear manifolds. Even though NPE and LPP provide linear embedding function, similar to the one obtained from PCA, these methods differ from PCA in the assumption of the linearity of the underlying manifold. This makes these methods capable of learning a non-linear manifold despite the linearity of $\hat{f}$. 
However, common practice in the Chemometrics literature is to validate their process control methods based on a single realization of a simulated chemical process, the Tennessee Eastman Process (described in section \ref{tep_results}) which limits the reliability of these results. Furthermore, the test statistic chosen by the authors is $T^2$ test statistic, which implicitly assumes that the lower-dimensional representation of the process generating data is i.i.d. and normally distributed, assumptions we did not observe in our experiments reported in Section \ref{tep_results}.

Although the Tennessee Eastman Process exhibits significant autocorrelation, prior work did not account for these temporal dependencies. It is well known that ignoring such dependencies can lead to misleading detection performance and underestimated false alarm rates. Moreover, the performance metric commonly used to validate prior methods was the fault detection rate, which is defined as the proportion of out-of-control signals given after a fault occurs, rather than the more widely adopted out-of-control average run length (ARL$_{out}$) performance metric. As a result, these methods did not provide a controllable ARL$_{in}$, despite the fact that achieving a desired ARL$_{in}$ is possible, as demonstrated in Section~\ref{tep_results}.

\section{Manifold Fitting Based SPC Framework}\label{manifold_fitting}

In this section, we present our main proposed method, which is based on monitoring deviations from the in-control manifold \( \mathcal{M} \). 

\subsection{Preliminary definitions }\label{dme_definitions}

We are interested in a $d$-dimensional Riemannian manifold $\mathcal{M}$ embedded in $\mathbb{R}^D$ with a positive reach $\operatorname{rch}(\mathcal{M})$. The reach of the manifold $\mathcal{M}$ is defined as \citep{federer1959curvature}: 
$$\operatorname{rch}(\mathcal{M}) = \inf_{x \in \mathcal{M}} \text{dist}(x,\text{Med}(\mathcal{M})),$$
where $\text{Med}(\mathcal{M})$ is the medial axis of $\mathcal{M}$ given by 
$$\text{Med}(\mathcal{M}) = \{ y \in \mathbb{R}^D \mid \exists x_1 \neq x_2 \in \mathcal{M}, \text{ such that }  \|x_1 - y \|=\|x_2 - y \| \}.$$
For example, a $d$-dimensional hyper sphere $\mathcal{S}^d=\{x\in \mathbb{R}^{d+1}, \quad  ||x||_2 = r\}$ has a reach equal to its radius $r$, and its medial axis consists of the origin in $\mathbb{R}^{d+1}$.

Now, let us recall our problem setup,
\begin{equation}\label{prob_definition}
    Y_t =
\begin{cases}
X_t + \mathcal{E}_t, & \text{for }  t = -m + 1, \dots, \tau -1 \\
X_t + \Delta + \mathcal{E}_t, & \text{for } t = \tau,\tau +1, \dots
\end{cases}
\end{equation}
where $\mathcal{E}_t \sim \mathcal{N}(0,\sigma^2I_D)$ and $X_t$ is supported on $\mathcal{M}$ whose reach $\operatorname{rch(\mathcal{M})}\gg \sigma$. Let the projection operator $\pi:\mathbb{R}^D\to \mathcal{M}$ be defined as
$$\pi(Y_t)=\arg \min_{x\in \mathcal{M}} \operatorname{dist}(Y_t,x)$$
which is well defined and continuous in the domain $\Gamma=\{y \in {\mathbb{R}^D}|\text{dist}(y,\mathcal{M})<\operatorname{rch}(\mathcal{M})\}$.

The distribution of deviations $||\pi(Y_t)-Y_t||_2$ depends both on the local geometry of $\mathcal{M}$ around $X_t$ and the noise. Under the assumption that $\operatorname{rch}(\mathcal{M})$ is very large, the influence of curvature on the distribution of $\|\pi(Y_t) - Y_t\|_2$ becomes negligible. 
Moreover, as shown in Appendix \ref{AppDeviationProof} (supp. materials), 
$$\mathbb{E}\left[\|\pi(Y_t) - Y_t\|_2^2 \mid t \geq \tau \right] > \mathbb{E}\left[\|\pi(Y_t) - Y_t\|_2^2 \mid t < \tau \right]$$
provided that $\operatorname{rch}(\mathcal{M}) = \infty$ and $\Delta$ is not entirely tangent to $\mathcal{M}$. Therefore, monitoring the deviations $\|\pi(Y_t) - Y_t\|_2$ provides a means of determining whether the process is in a state of statistical control. This approach effectively reduces a high-dimensional SPC problem to the monitoring of a scalar statistic. Since the projection operator $\pi(\cdot)$ is unknown, we will carry out this approach with an estimate $\hat{\pi}(\cdot)$ obtained from Phase I samples. 


\subsection[Fitting the Manifold M]{Fitting the Manifold $\mathcal{M}$}

\cite{yao2023manifold} describe a method for fitting a manifold from point data. Given $m = O(\sigma^{-(d+3)})$ observations, this approach fits a $d$-dimensional smooth manifold $\hat{\mathcal{M}}$ in $\mathbb{R}^D$, achieving an estimation error, measured by the Hausdorff distance, bounded by $O(\sigma^2 \log (1/\sigma))$. The bound depends on the intrinsic dimensionality $d$ and the noise level $\sigma$.  In addition to offering a tight error bound, the method does not require prior knowledge of intrinsic dimensionality $d$, in contrast to many manifold learning techniques.

This manifold fitting method consists of two main steps. In the first step, the objective is to estimate the contraction direction defined as the direction of $Y_t-\pi(Y_t)$. To do this, we first construct a $D$-dimensional Euclidean Ball centered at $Y_t$ with radius $r_0$, denoted by $\mathcal{B}_D(Y_t,r_0)$. We then estimate the contraction direction by taking a weighted average of the points within this ball. The algorithm for this step is outlined below.
\begingroup
\spacingset{1.2}
\begin{algorithm}[htbp!]
\caption{Estimate Contraction Direction}
\label{alg:contraction_direction}
\begin{algorithmic}[1]
\REQUIRE Observations  $\{Y_t\}_{t = -m+1}^{0} \subset \mathbb{R}^D$, point $z \in \mathbb{R}^D$ to be projected , radius $r_0 = C_0 \sigma$, weight exponent $k > 2$ to guarantee a twice-differentiable smoothness.
\ENSURE Estimated direction vector $\mu_z - z$

\FOR{each \( t \in \{-m+1,\dots,0\} \)}
    \STATE $\tilde{\alpha}_t = \begin{cases} \left(1 - \frac{\|Y_t - z\|_2^2}{r_0^2} \right)^k &, \text{if } \|Y_t-z\|_2 \leq r_0 \\ 0  &, \text{otherwise }  \end{cases} $
\ENDFOR
\STATE Normalize weights: $\alpha_t = \tilde{\alpha}_t / \sum \tilde{\alpha}_t$
\STATE Compute weighted average: $\mu_z = \sum \alpha_t Y_t$
\RETURN  $\mu_z - z$
\end{algorithmic}
\end{algorithm}
\endgroup
Algorithm \ref{alg:contraction_direction} returns a direction $\mu_z - z$ that aligns with the projection direction $\pi(z) - z$. More precisely, if $m = C_3 \sigma^{-(d+3)}$ and $\operatorname{dist}(z,\mathcal{M})< C\sigma$, it follows from Theorem 3.5 in \cite{yao2023manifold} that, 
$$\sin\left\{ \Theta(\mu_z - z, \pi(z) - z) \right\} = O\bigg(\sigma \sqrt{\log(1/\sigma)}\bigg).$$
with probability at least $1-O\big(\exp{(-c_0\sigma^{-c_1})}\big)$ for some constants $c_0$ and $c_1$.
In the second step of the method, the objective is to move along this contraction direction by constructing a hyper-cylinder defined as
$$V_{Y_t} = \mathcal{B}_{D-1}(Y_t,r_1)\times \mathcal{B}_{1}(Y_t,r_2)$$
where the axis is aligned with the direction estimated in Algorithm~\ref{alg:contraction_direction} and $r_1=C_2\sigma$, $r_2=C_2\sigma \sqrt{\log (1/\sigma)}$. The projected point $\pi(Y_t)$ is estimated by computing a weighted average of the points contained within a hyper-cylinder. This step in the manifold fitting method can be interpreted as placing a thin, elongated tube along the estimated contraction direction to guide the estimation of the projected point. The algorithmic formulation of this step is presented below.
\begingroup
\spacingset{1.2}
\begin{algorithm}[htbp!]
\caption{Local Contraction Toward the Manifold}
\label{alg:contract_point}
\begin{algorithmic}[1]
\REQUIRE Observations $\{Y_t\}_{t = -m+1}^{0} \subset \mathbb{R}^D$, point $z \in \mathbb{R}^D$ to be projected , radius parameters $r_1=C_2\sigma$, $r_2=C_2\sigma \sqrt{\log(1/\sigma)}$, weight exponent $k\geq 2$
\ENSURE Projected point $\hat{\pi}(z)$

\STATE Compute the estimated direction $\mu_z - z$ with the Algorithm \ref{alg:contraction_direction}
\FOR{each \( t \in \{ -m +1,\dots,0\} \)}
    \STATE Compute:
    $u_t = \frac{(\mu_z - z)(\mu_z - z)^{\top}}{\|\mu_z - z\|_2^2}(Y_t-z), \quad v_t = Y_t - z - u_t$
    {\footnotesize
    \[
    w_v(v_t) = \begin{cases} \left(1 - \frac{\|v_t\|_2^2}{r_1^2} \right)^k, & \text{if } \|v_t\|_2 \leq r_1  \\  
    0, & \text{otherwise } \end{cases},
    \quad
    w_u(u_t) = \begin{cases}
    1, & \text{if } \|u_t\|_2 \leq r_2 / 2 \\
    \left(1 - \left( \frac{2\|u_t\|_2 - r_2}{r_2} \right)^2 \right)^k, & \text{if } \|u_t\|_2 \in (r_2/2, r_2) \\
    0, & \text{otherwise}
    \end{cases}
    \]
    }
    Set $\tilde{w}_t = w_u(u_t) \cdot w_v(v_t)$
\ENDFOR
\STATE Normalize weights: $\beta_t = \tilde{w}_t / \sum \tilde{w}_t$
\STATE Compute weighted average: $\hat{\pi}(z) = \sum \beta_t Y_t$
\RETURN \( \hat{\pi}(z) \)
\end{algorithmic}
\end{algorithm}
\endgroup
Algorithm~\ref{alg:contract_point} returns a point $\hat{\pi}(z)\in \mathbb{R}^D$ that estimates the projection $\pi(z)$. Similarly, if \( m = C_3 \sigma^{-(d+3)} \) and $\operatorname{dist}(z,\mathcal{M})< C\sigma$, then it follows from Theorem 4.5 in~\cite{yao2023manifold} that,
$$\|\hat{\pi}(z) - \pi(z)\|_2 =O\bigg( \sigma^2 \log(1/\sigma)\bigg),$$
with probability at least  $1-O\big(\exp{(-c_0\sigma^{-c_1})}\big)$  for some constant $c_0$ and $c_1$.

Now, suppose that $m$ Phase I observations are given. These points naturally define a region $\Gamma=\{y \in {\mathbb{R}^D}|\text{dist}(y,\mathcal{M})<C\sigma\}$. Then, it follows from Theorem 5.2 in \cite{yao2023manifold} that $\hat{\mathcal{M}}=\{\hat{\pi}(y) \mid y \in \Gamma\}$ is a smooth sub-manifold of $\mathbb{R}^D$ s.t.
for any $x \in \hat{\mathcal{M}}$, $\operatorname{dist}(x,\mathcal{M})\leq O\big(\sigma^2 \log(1/\sigma)\big)$ with high probability. 

The algorithm presented in \cite{yao2023manifold} achieves a state-of-the-art error bound within the manifold fitting literature, but the constants $C_0,C_1$ and $C_2$ were not defined, nor a method to estimate $\sigma$ from the data was given, and these determine the distances $r_0$, $r_1$, and $r_2$ in Algorithm 1-2. 
Algorithm 3 provides a practical estimate of $\sigma$, and we determine the constants $r_0$, $r_1$, and $r_2$ that ensure the manifold fitting method theoretical guarantees by stipulating that:
\begin{enumerate}[label=(\alph*)]
    \item $r_0,r_1,r_2 \ll \operatorname{rch}(\mathcal{M})$ and $r_2\geq r_0 \geq r_1$
    \item $r_0,r_1,r_2$ are large enough that we have enough number of points, e.g. 5 points, in $\mathcal{B}_D(y,r_0)$ and  $V_{y}$ for any $y \in \Gamma$
    \item $r_0$ is large enough that $\mathcal{B}_D(y,r_0) \cap \mathcal{M} \neq \emptyset$ for any $y \in \Gamma$\vspace{-0.1cm}
\end{enumerate}
\begingroup
\spacingset{1.2}
\begin{algorithm}[htbp!]
\caption{Iterative Estimation of Noise Level \( \sigma \)}
\label{alg:estimate_sigma_residual}
\begin{algorithmic}[1]
\REQUIRE Observations $\{Y_t\}_{t = -m+1}^{0} \subset \mathbb{R}^D$, manifold dimension $d$, initial guess $ \sigma^{(0)}$, tolerance $\epsilon>0$, max iterations  $T$
\ENSURE Estimated noise level $\hat{\sigma}$
\STATE Initialize iteration index $i = 0$, set $\sigma^{(0)}$
\REPEAT
    \STATE Set \( \sigma = \sigma^{(i)} \)
    \STATE Compute projection $\hat{\pi}(Y_t)$ for all  $t$ using Algorithms~\ref{alg:contraction_direction} and~\ref{alg:contract_point} with current $\sigma$
    \STATE Update estimate:
    $\sigma^{(i+1)} = \sqrt{ \frac{1}{m(D - d)} \sum_{t = -m+1}^0 \|Y_t - \hat{\pi}(Y_t)\|_2^2 }$
    \STATE Increment $i \leftarrow i + 1$
\UNTIL{ \( |\sigma^{(i)} - \sigma^{(i-1)}| < \epsilon \) or \( i \ge T \) }
\RETURN \( \hat{\sigma} = \sigma^{(i)} \)
\end{algorithmic}
\end{algorithm}
\endgroup
Algorithm 3 ignores any curvature effects of $\mathcal{M}$ but it remains practical when $\operatorname{rch}(\mathcal{M})$ is large. Also, if $D\gg d$, we can simply neglect $d$ and divide the sum of squared residuals by $D$. We emphasize that the objective here is not to obtain an unbiased estimator of $\sigma$, but rather to obtain a practical estimate since any deviation from the true $\sigma$ will ultimately be absorbed into the constants $C_0$, $C_1$ and $C_2$.

\subsection[Illustrative Example: a stochastic process on a d-sphere]%
{Illustrative Example: a stochastic process on a $d$-sphere}\label{sec:illustraive_examples}

Let $\mathcal{M}$ be a $d$-dimensional unit sphere $\mathcal{S}^d$ embedded in $\mathbb{R}^{D}$ for $d<D$, defined as $\mathcal{S}^d = \{ x \in \mathbb{R}^D : x_{d+2} = \cdots = x_D = 0, \; \sum_{i=1}^{d+1} x_i^2 = 1 \}$. Suppose the latent process $\{X_t\}$ evolves on $\mathcal{S}^d$ according to the following relation:
\begin{equation}\label{eq:synthetic_latent}
    X_t = \pi(X_{t-1} + E_{t}), \quad \text{for} \quad t=-m+1,-m+2,\dots,
\end{equation}
where $E_t \sim \mathcal{N}(0, \sigma_x^2 I_{D})$ is independent Gaussian noise, $X_{-m}\sim \text{Unif}(\mathcal{S}^d)$ and $\pi(\cdot)$ denotes the projection operator onto $\mathcal{S}^d$, ensuring that $X_t \in \mathcal{S}^d$ for all $t$. This stochastic process is strictly stationary with a uniform stationary distribution on $\mathcal{S}^d$ as shown in Appendix \ref{AppStcProcess} (supp. materials).
The observed process $\{Y_t\}$ is then defined as in our main problem setting:
\begin{equation}\label{eq:synthetic_ambient}
Y_t = X_t + \mathcal{E}_t, \quad \text{for} \quad t= -m+1,-m+2,\dots,
\end{equation}
where \( \mathcal{E}_t \sim \mathcal{N}(0, \sigma^2 I_D) \) is independent observational noise. 
\begin{figure}[H]
 \centering
  \includegraphics[width=0.8\textwidth]{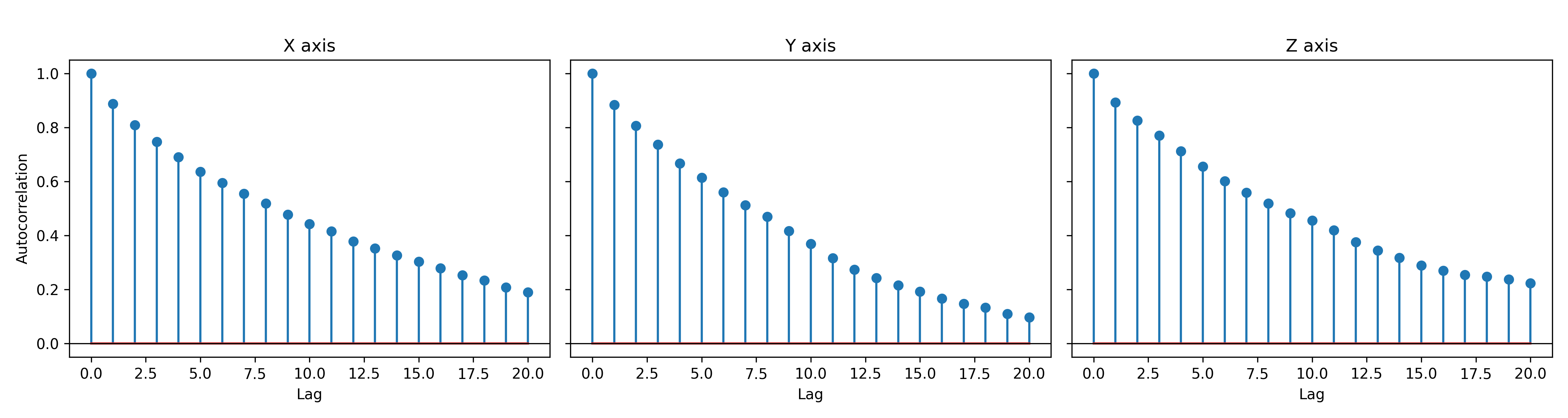}
 \caption{Estimated autocorrelation of $\{Y_t\}$ along each axis. Unlike a stationary process in Euclidean space, the estimated autocorrelations do not decay exponentially. We set $\sigma_x = 0.3$ and $\sigma = 0.1$ to generate this single realization.}
  \label{fig:auto_correlation_on_sphere}
 \end{figure}
We generated a single realization of $Y_t$ with $d=2$ and $D=3$ for visualization purposes over 1100 time steps. Figure \ref{fig:auto_correlation_on_sphere} shows the estimated autocorrelation of the process in each axis. 
\begingroup
\spacingset{1.2}
 \begin{figure}[htbp!]
    \centering
    \begin{subfigure}[b]{0.3\textwidth}
        \centering
        \includegraphics[width=\textwidth]{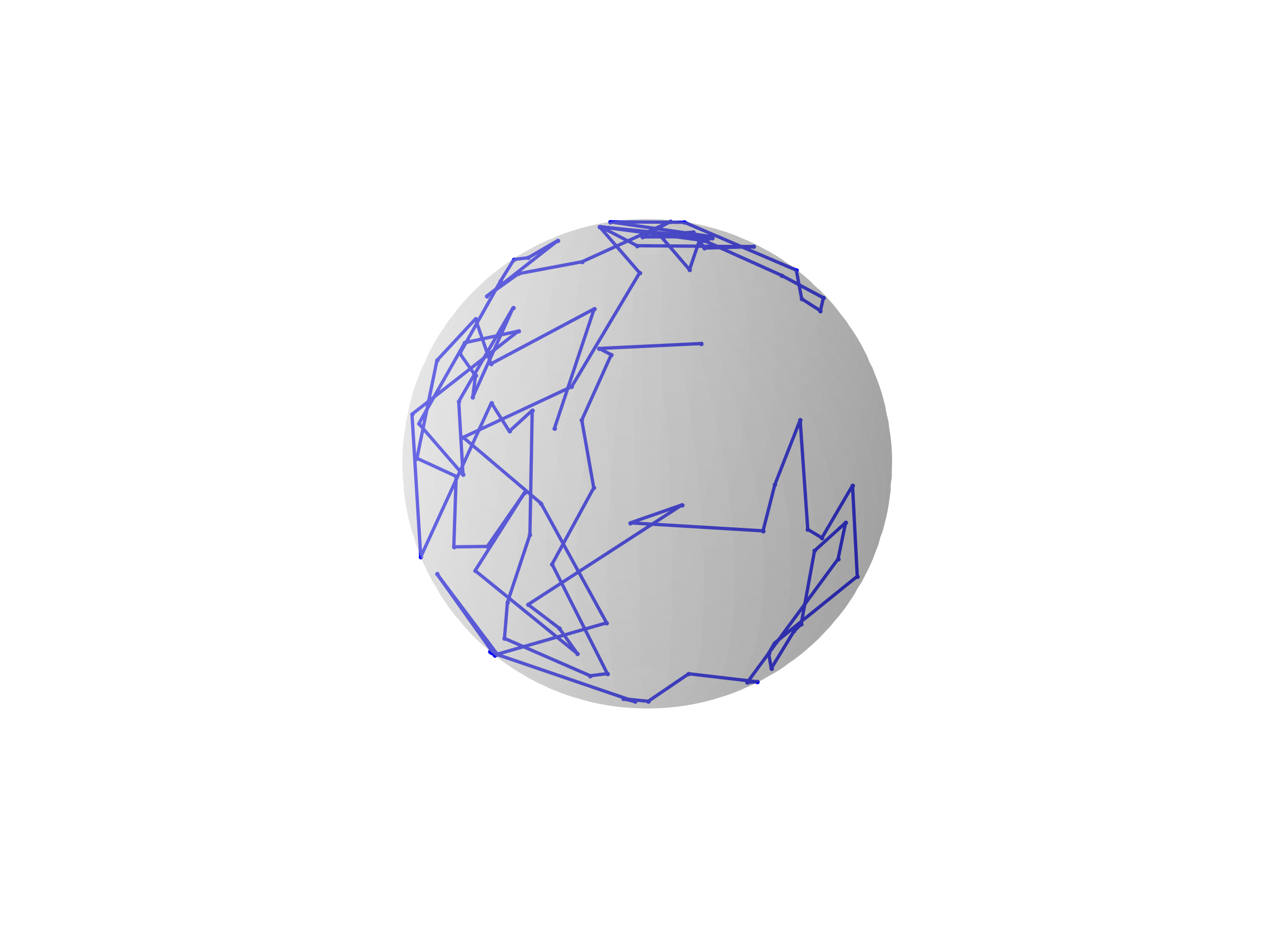}
        \caption{Trajectory of $X_t\in \mathcal{M}$}
    \end{subfigure}
    \hfill
    \begin{subfigure}[b]{0.3\textwidth}
        \centering
        \includegraphics[width=\textwidth]{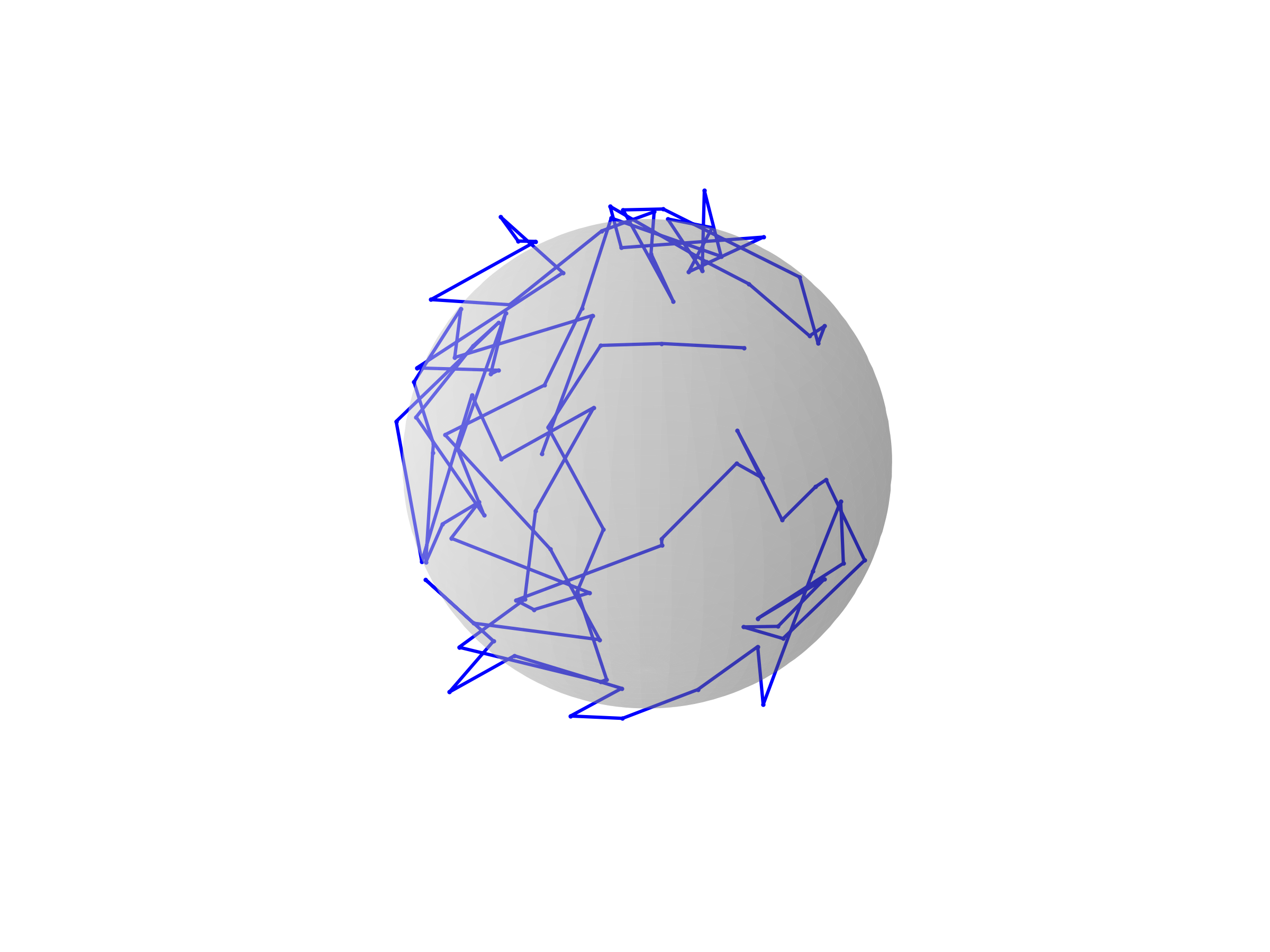}
        \caption{Trajectory of $Y_t\in \mathbb{R}^3$}
    \end{subfigure}
        \hfill
    \begin{subfigure}[b]{0.3\textwidth}
        \centering
        \includegraphics[width=\textwidth]{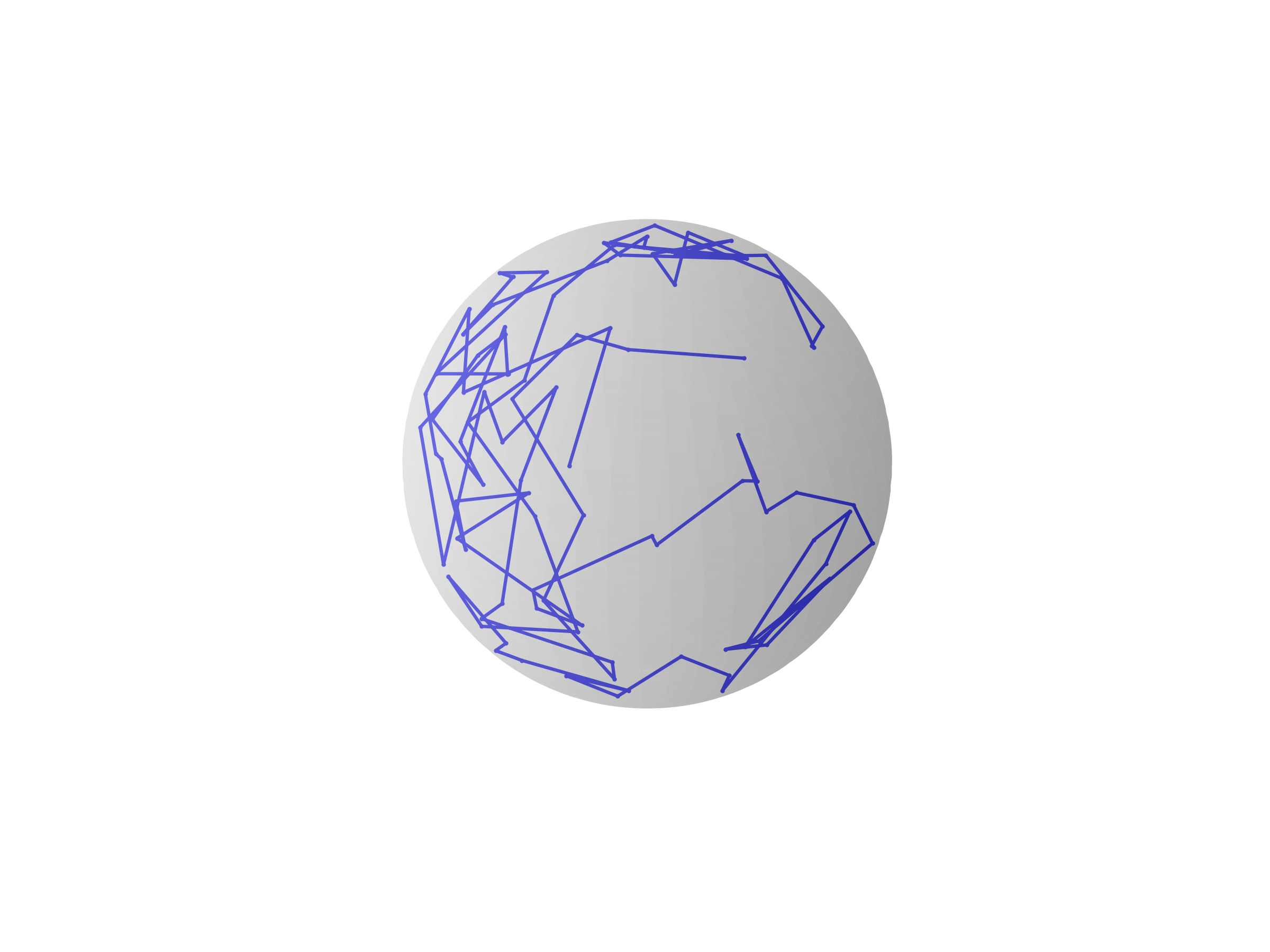}
        \caption{Trajectory of $\hat{\pi}(Y_t)$}
    \end{subfigure}
    \hfill
    \caption{The trajectories of $X_t \in \mathcal{M}$, $Y_t \in \mathbb{R}^3$, and the projections onto the fitted manifold $\hat{\pi}(Y_t)$ of the last 100 observations. We set \( C_0 = C_2 = 4 \) and \( C_1 = 2 \) in Algorithms~\ref{alg:contraction_direction}, \ref{alg:contract_point}, and \ref{alg:estimate_sigma_residual}. Algorithm~\ref{alg:estimate_sigma_residual} is initialized with \( \sigma^{(0)} = 0.05 \), yielding an estimated value of \( \hat{\sigma} = 0.1298 \).}
    \label{fig:process_on_sphere}
\end{figure}
\endgroup
The first 1000 observations were used to fit the manifold in $\mathbb{R}^3$. Subsequently, the remaining 100 observations were projected onto the fitted manifold.

\begingroup
\spacingset{1.2}
 \begin{figure}[htbp!]
    \centering
    \begin{subfigure}[b]{0.3\textwidth}
        \centering
        \includegraphics[height=3cm]{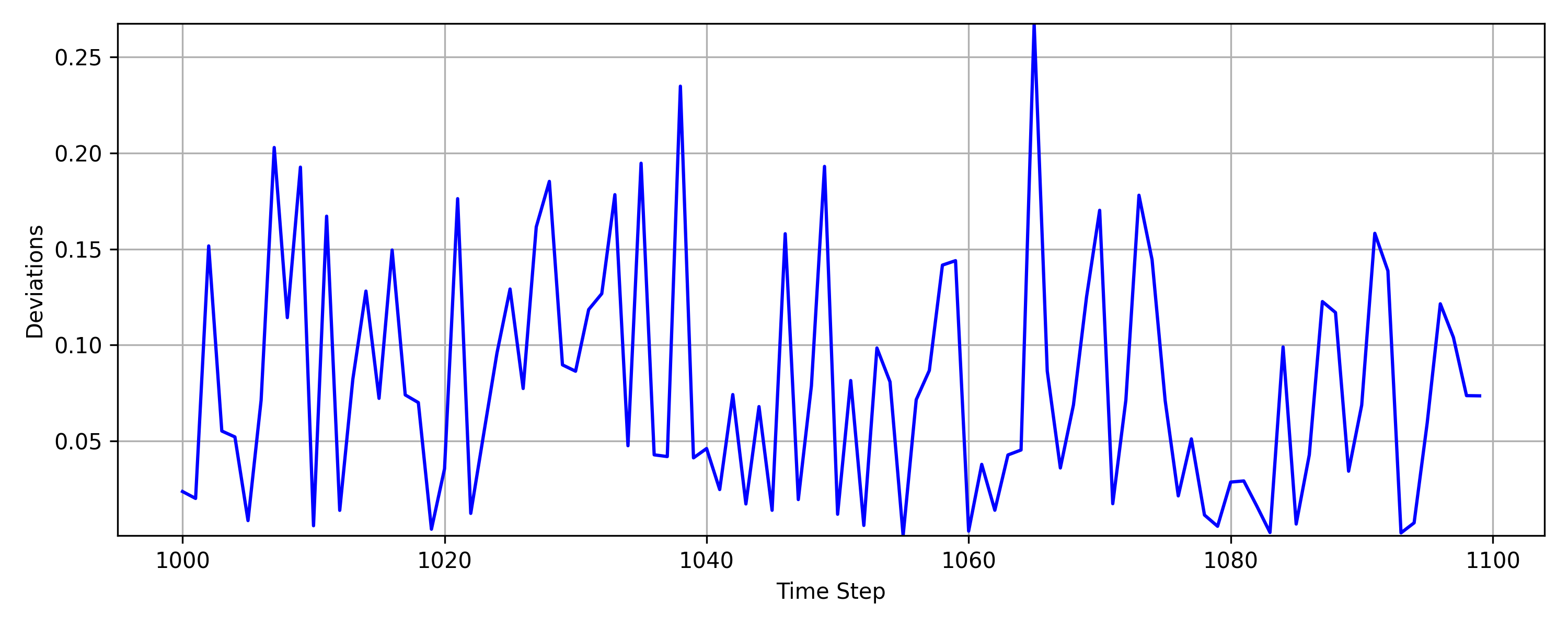}
        \caption{$||Y_t - \pi(Y_t)||_2$}
    \end{subfigure}
    \hfill
    \begin{subfigure}[b]{0.3\textwidth}
        \centering
        \includegraphics[height=3cm]{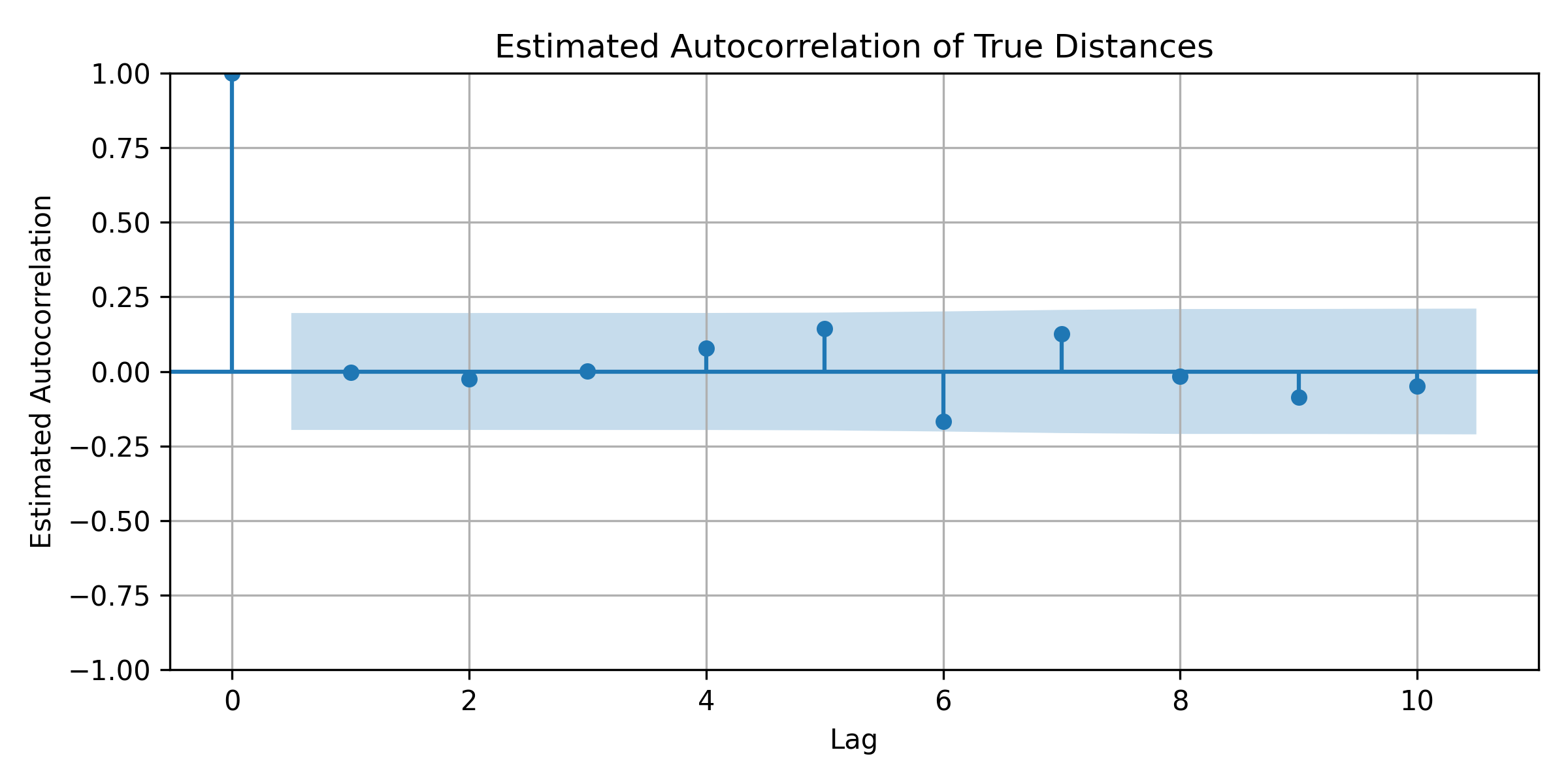}
        \caption{Estimated autocorrelations}
    \end{subfigure}
    \hfill
    \begin{subfigure}[b]{0.3\textwidth}
        \centering
        \includegraphics[height=3cm]{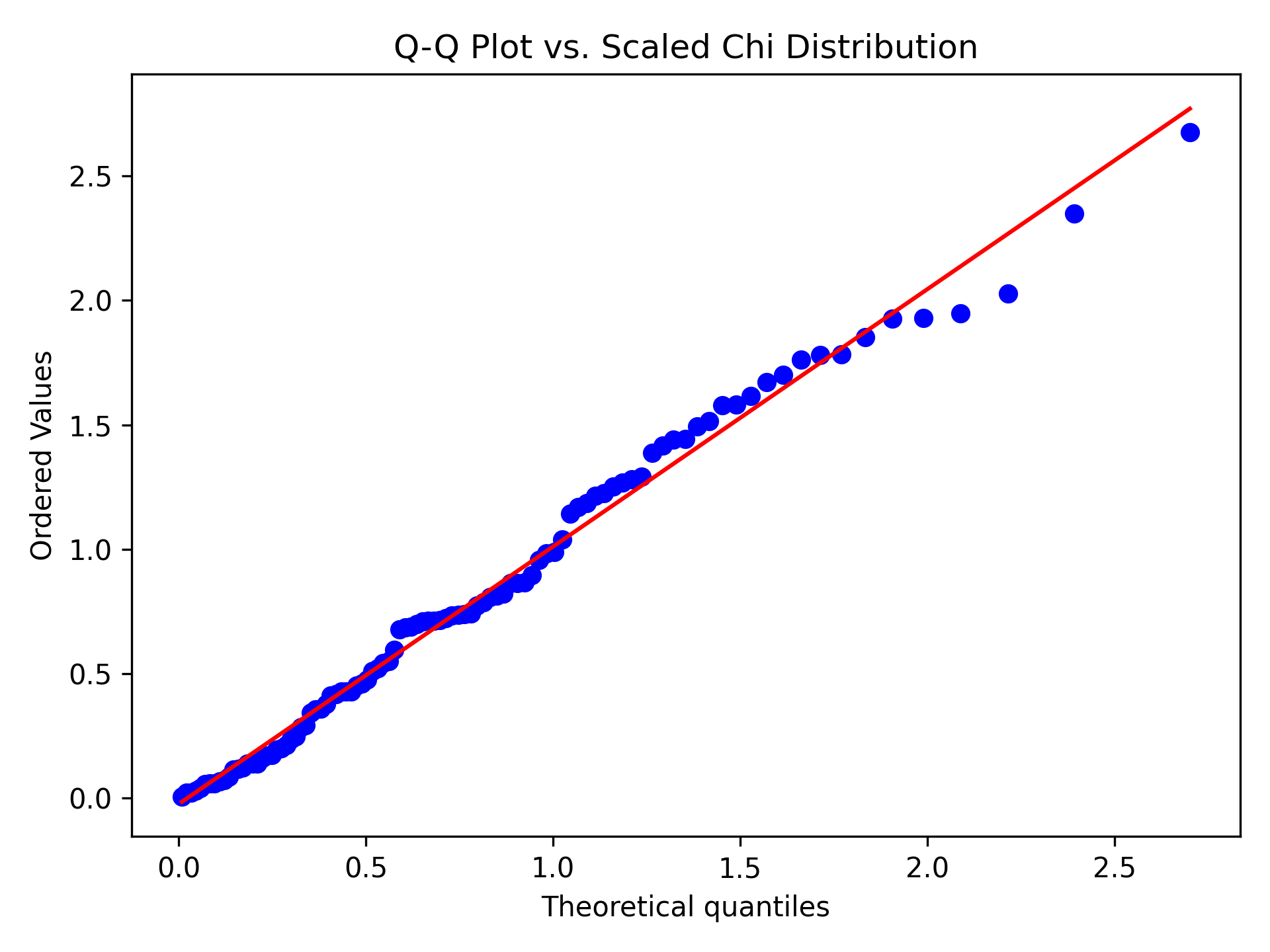}
        \caption{Q-Q Plot of $\sigma\chi_{1}$ vs $||Y_t - \pi(Y_t)||_2$}
    \end{subfigure}

    \vspace{0.4cm}
    \begin{subfigure}[b]{0.3\textwidth}
        \centering
        \includegraphics[height=3cm]{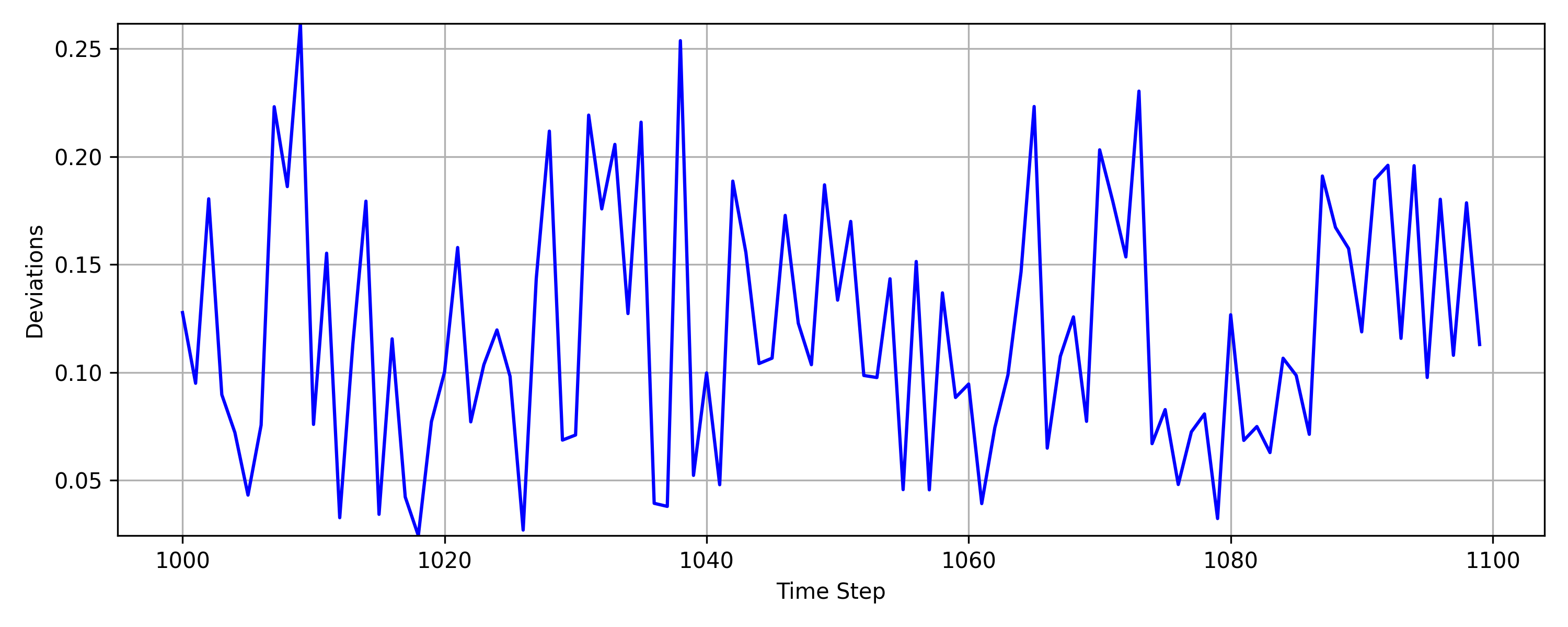}
        \caption{$||Y_t - \hat{\pi}(Y_t)||_2$}
    \end{subfigure}
    \hfill
    \begin{subfigure}[b]{0.3\textwidth}
        \centering
        \includegraphics[height=3cm]{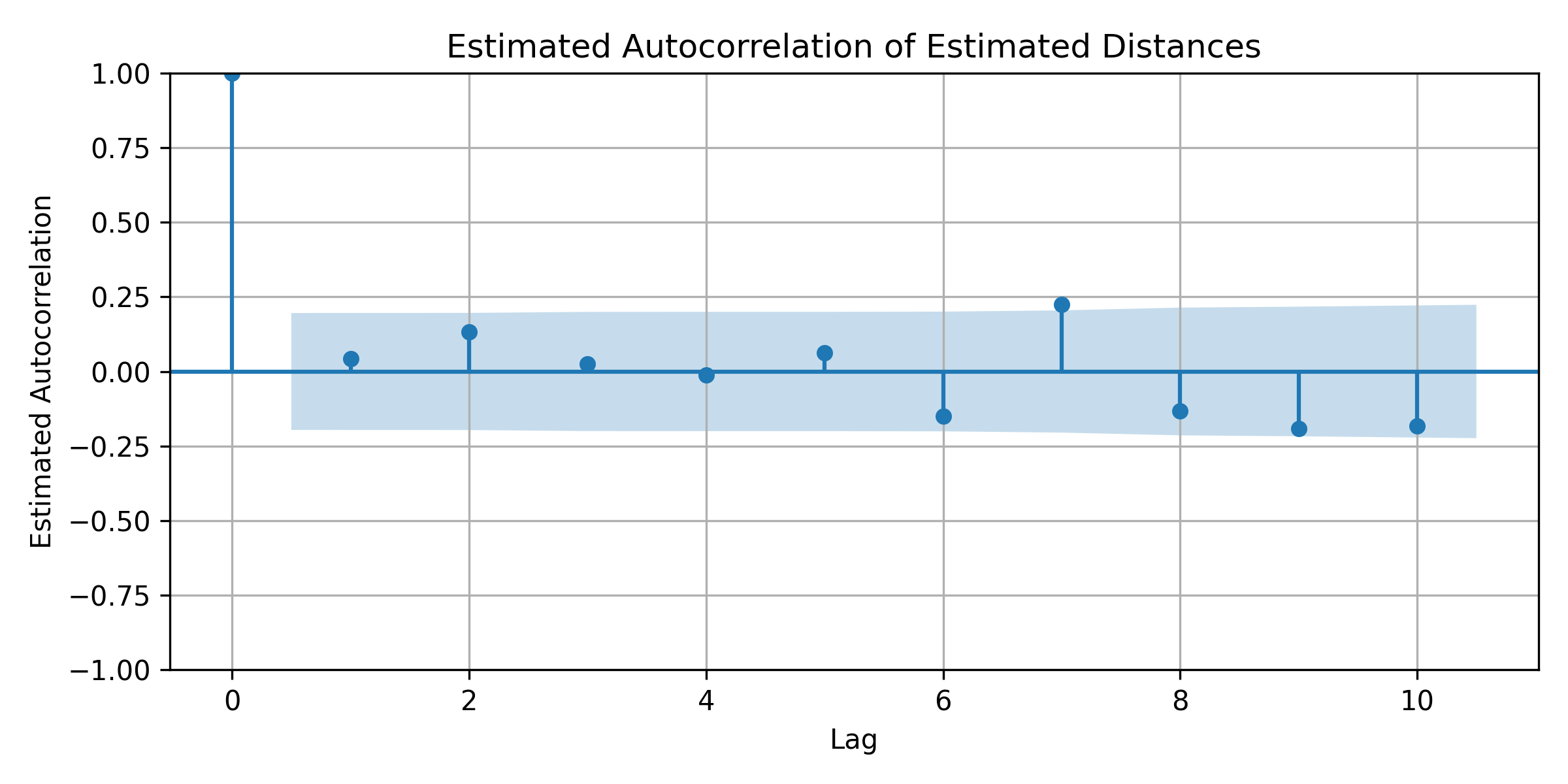}
        \caption{Estimated autocorrelations}
    \end{subfigure}
    \hfill
    \begin{subfigure}[b]{0.3\textwidth}
        \centering
        \includegraphics[height=3cm]{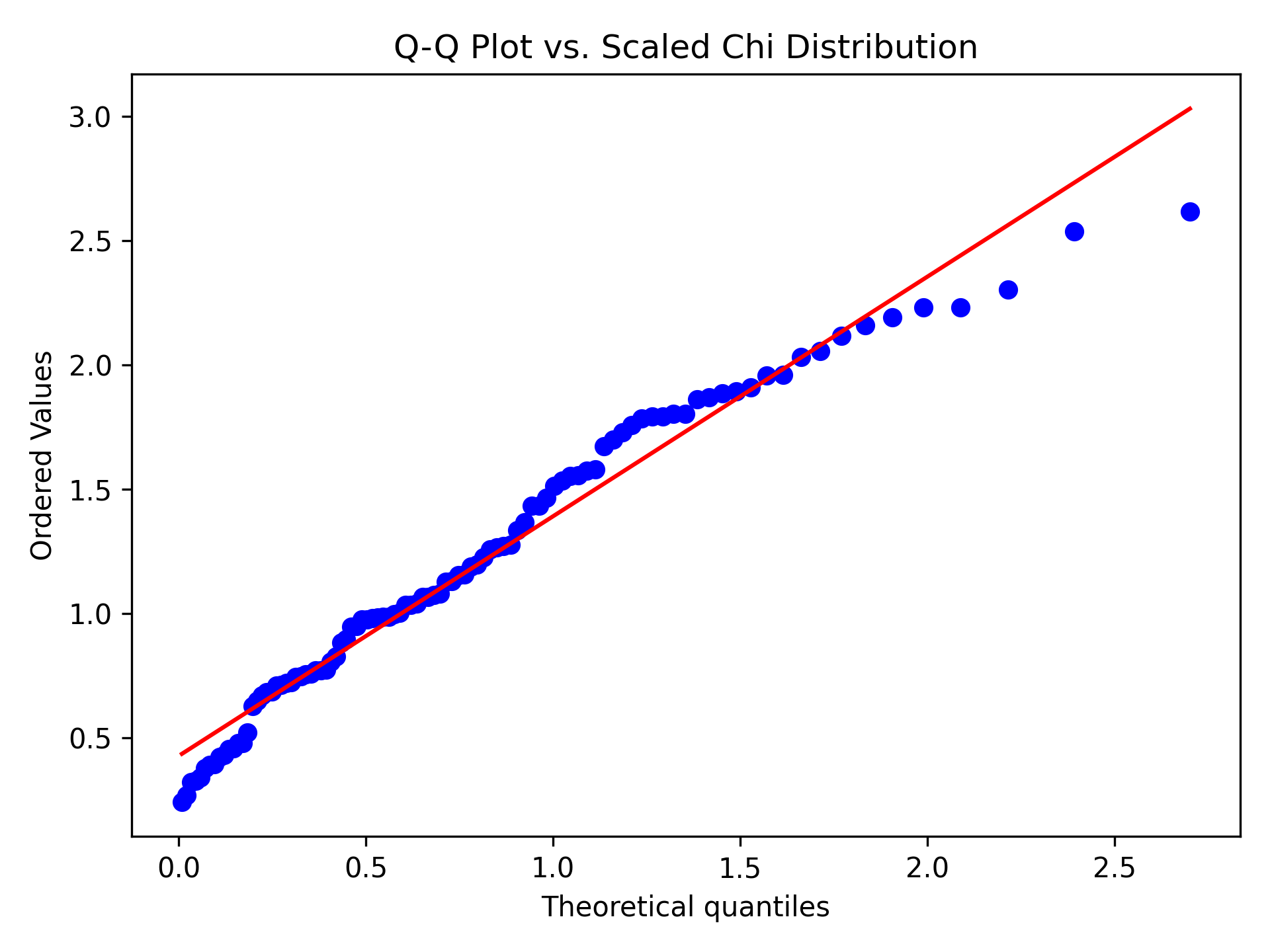}
        \caption{Q-Q Plot of $\sigma\chi_{1}$ vs $||Y_t - \hat{\pi}(Y_t)||_2$}
    \end{subfigure}
    \caption{Top row: (a) last 100 true deviations from the 2-sphere manifold, i.e.,  \(||Y_t - \pi(Y_t)||_2\), (b) estimated autocorrelation of the true deviations, (c) Q–Q plot comparing \(\sigma\chi_{1}\) to \(||Y_t - \pi(Y_t)||_2\).  Bottom row: (d) last 100 estimated deviations from the manifold, i.e. $||Y_t - \hat{\pi}(Y_t)||_2$,  (e) their estimated autocorrelation, (f) Q–Q plot comparing \(\sigma\chi_{1}\) to the estimated deviations \(||Y_t - \hat{\pi}(Y_t)||_2\). The true deviations are i.i.d. as shown in Appendix \ref{AppStcProcess} (supp. materials) and match a $\sigma\chi_{1}$ distribution as confirmed by the Kolmogorov–Smirnov (KS) test (p-value = 0.80). In contrast, the estimated deviations show mild autocorrelation and deviate from a $\sigma\chi_{1}$ distribution (KS p-value of 0.0.). }
    \label{fig:distances_summary}
\end{figure}
\endgroup
Figure~\ref{fig:process_on_sphere} shows the trajectories $\{X_t\}$, $\{Y_t\}$, and $\{\hat{\pi}(Y_t)\}$ for these final 100 steps. Although $\{Y_t\}$ lies outside $\mathcal{M}$, its projection $\{\hat{\pi}(Y_t)\}$ remains close to $\mathcal{M}$. 
Figure \ref{fig:distances_summary} shows that $||Y_t-\hat{\pi}(Y_t)||_2$ do not follow a $\sigma\chi_{1}$ distribution in contrast to the true distances. This justifies the need for a distribution-free control chart for monitoring the estimated deviations from the fitted manifold.
\begingroup
\spacingset{1.2}
 \begin{figure}[htbp!]
    \centering
    \begin{subfigure}[b]{0.45\textwidth}
        \centering
        \includegraphics[width=\textwidth]{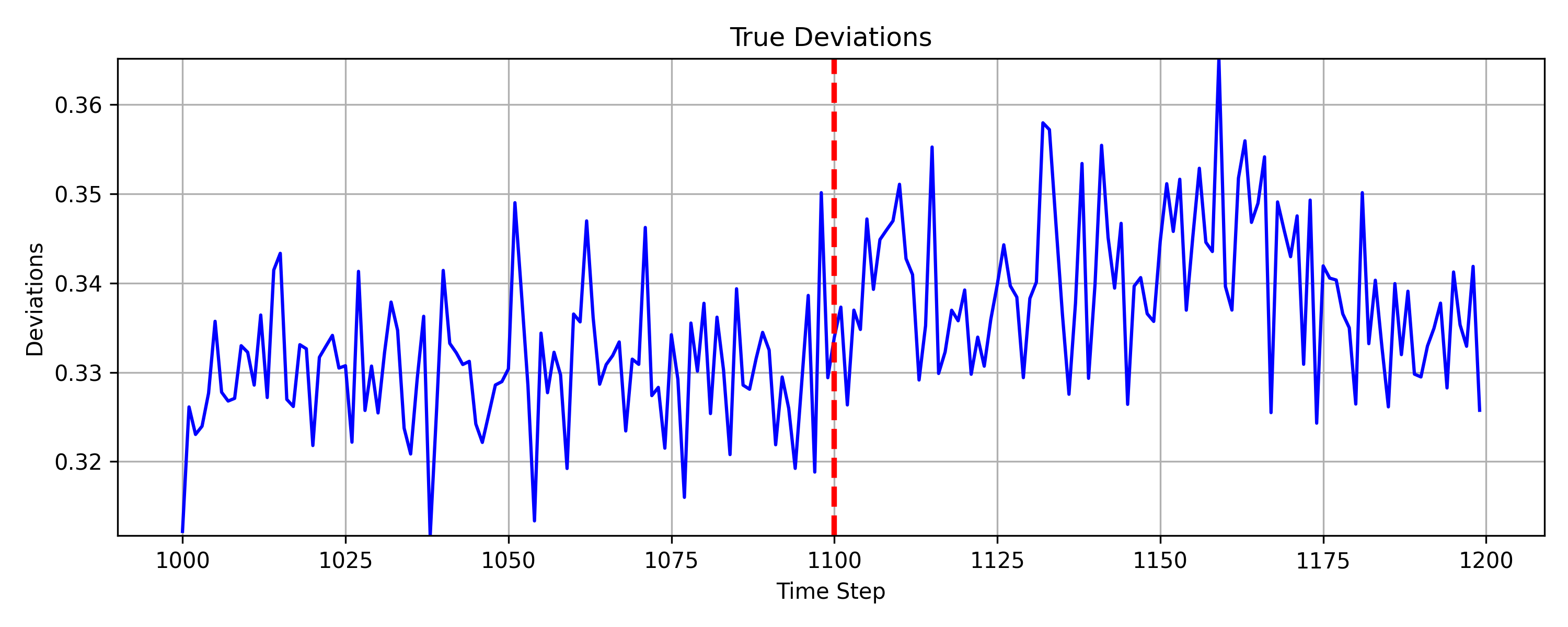}
        \caption{True deviations of the last 200 observations}
    \end{subfigure}
    \hfill
    \begin{subfigure}[b]{0.45\textwidth}
        \centering
        \includegraphics[width=\textwidth]{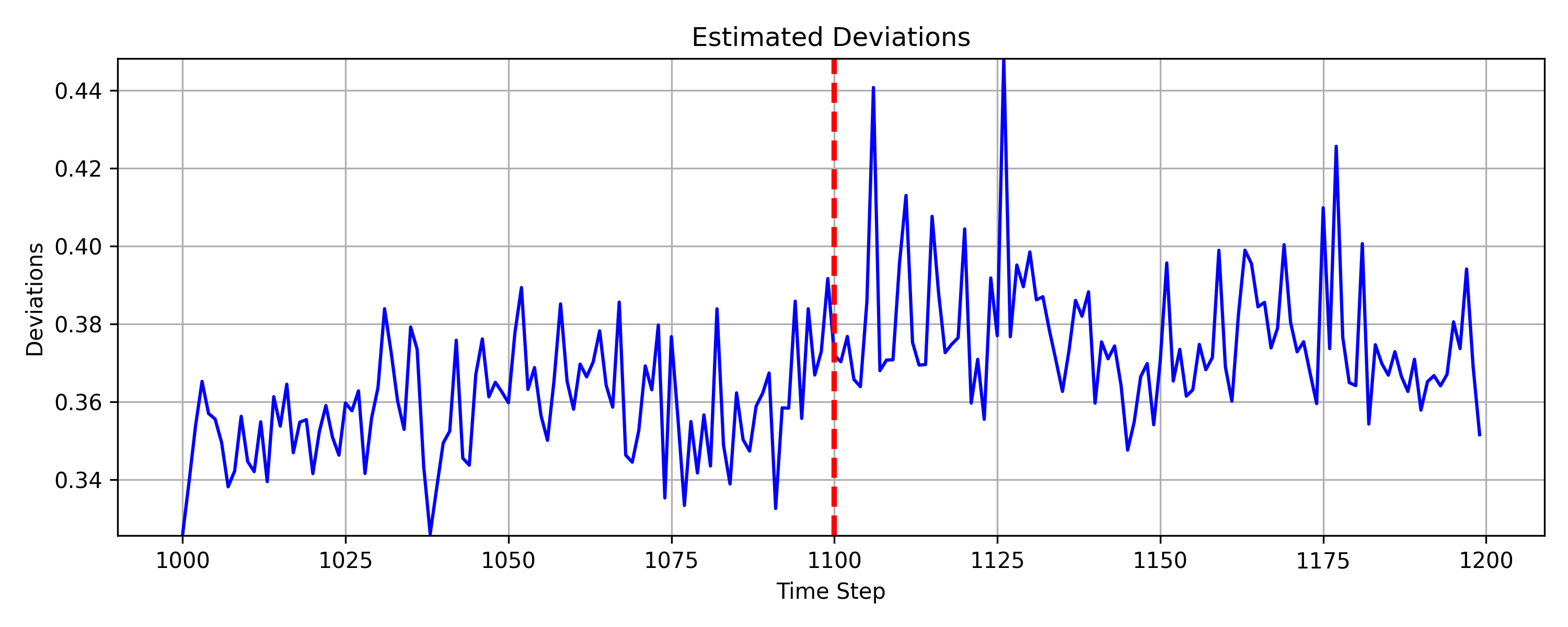}
        \caption{Estimated deviations of the last 200 observations}
    \end{subfigure}
    \hfill
    \caption{True and estimated deviations from the 2-sphere manifold for the last 200 observations, where $\sigma_x = 0.3$, $\sigma = 0.01$. We set $C_0 = C_2 = 50$, and $C_1 = 35$ in Algorithms ~\ref{alg:contraction_direction}, \ref{alg:contract_point}, and \ref{alg:estimate_sigma_residual}. Similarly, Algorithm \ref{alg:estimate_sigma_residual} was initialized with $\sigma^{(0)} = 0.05$ and yielded an estimated value of $\hat{\sigma} = 0.01073$ . The vertical dashed line indicates the time point at which the mean shift is introduced. Both distances increase after the $1100^{\text{th}}$ observation due to the applied mean shift.}
\label{fig:high_dimensional_process_on_sphere}
\end{figure}
\endgroup
Next, we consider an extreme case where \( m < D \). We generated a single realization of \( Y_t \) with $d=2$ and $D=1100$ over 1200 time steps, of which first $m=1000$ is used to fit a manifold in $\mathbb{R}^D$. After the \( 1100^{\text{th}} \) observation, we introduce a mean shift \( \Delta \in \mathbb{R}^{1100} \), such that the first component of \( \Delta \) equals \( 10\sigma \), while all other components are zero. Figure~\ref{fig:high_dimensional_process_on_sphere} shows the true and estimated deviations for the last 200 observations. Both increase after \( 1100^{\text{th}} \) observation due to the mean shift.

\subsection{A Distribution-free Univariate Control Chart for the Deviations from Manifold}\label{control_chart}

In this subsection, we construct a distribution-free EWMA (Exponentially Weighted Moving Average) control chart using a rolling window of size $w$ and a weighting parameter $\lambda \in (0,1)$ to monitor the estimated deviations from the manifold $\mathcal{M}$. We refer to this new chart as the Univariate Deviation From Manifold (UDFM) control chart. It is inspired on the multivariate ``DFEWMA" chart developed by \cite{chen2016distribution}. While the DFEWMA chart is intended to detect mean shifts in any direction, our setting only requires detecting deviations away from $\mathcal{M}$, making negative mean shifts irrelevant. Consequently, we modify the test statistic from \cite{chen2016distribution} (see also \cite{zhao2021intrinsic} who corrected this statistic) to improve the sensitivity of the UDFM control chart to positive mean shifts. Although the estimated deviations may display slight temporal dependencies, as shown in the previous subsection, we will still formulate UDFM under the assumption of i.i.d. observations, since the estimated deviations can easily be prewhitened if necessary.

We will first describe a two-sample test. Let $V_{-m+1}, \dots, V_{0} \overset{\text{i.i.d.}}{\sim} G$ and $U_1, \dots, U_n \overset{\text{i.i.d.}}{\sim} H$, where $G$ and $H$ are continuous distributions, and the two samples are independent. The hypothesis being tested is:
$H_0: \delta = 0 \quad \text{vs.} \quad H_1: \delta > 0$,
where we assume $H(x) = G(x - \delta)$ for all $x \in \mathbb{R}$, with $\delta > 0$ under $H_1$.
Let $N = m + n$, and let $S_{m,n} := \{U_{-m+1}, \dots, U_0, V_1, \dots, V_n\}$ be the pooled sample. Denote the ranks by $R_j$ for $j\in \{-m+1,...,n\}$ in the pooled sample. Define
$$ Z_j := \frac{1}{m+n} \max\left(0, R_j - \frac{m+n+1}{2} \right),$$
and $\mu_N=\mathbb{E}[Z_j | H_0]$ and $\sigma^2_N = \operatorname{Var}(Z_j | H_0)$. Then, the test statistic is defined as
$$T_{m,n} := \frac{\sum_{j=1}^{n}( Z_j - \mu_N)}{\sigma_N\sqrt{n(1-\frac{n-1}{m+n-1})}},$$
Under $H_0$, the exact distribution of $T_{m,n}$ can be determined since $Z_j$ depends only on $R_j$, whose distribution is known. However, we instead design a permutation test, as it is otherwise difficult to keep track of the distribution of $T_{m,n}$ on an EWMA control chart with our charting procedure explained later. Thus, let $T^*_{m,n}$ denote the test statistic obtained from uniformly permuted labeling of $S_{m,n}$.
\begin{theorem}
Let $c_{m,n}(\alpha)$ be the $\alpha$ upper quantile of the distribution of $T^*_{m,n}$ given $S_{m,n}$. Suppose $m$ and $n$ are large enough that the exact $\alpha$ upper quantile is well defined. Then,
\begin{enumerate}[label=(\alph*)]
    \item Under $H_0$, $\mathbb{P}\bigg(T_{m,n} \geq c_{m,n}(\alpha) \bigg) = \alpha$.
    \item Under $H_1$, $\mathbb{P}\bigg(T_{m,n} \geq c_{m,n} (\alpha)\bigg) \to 1$ if $m,n \to \infty$ and $\frac{n}{N}\to \gamma \in (0,1)$
\end{enumerate}
\end{theorem}
\begin{proof} See Appendix \ref{AppTheorem1} (supp. materials).
\end{proof}

We now describe how to sequentialize the two-sample test as an EWMA control chart. Suppose $m$ in-control observations are available, followed by $n$ additional observations. Let the charting statistic \( T_{n,w,\lambda} \), with a window size $w$ and a weighting parameter $\lambda\in (0,1)$, be defined at the $n^{th}$ observation as
$$T_{n,w,\lambda} = \frac{\sum_{j=n-w+1}^{n}(1-\lambda)^{n-j}\left( Z_{j,n} - \mathbb{E}[Z_{j,n}] \right)}{\sqrt{\operatorname{Var}\left( \sum_{j=n-w+1}^{n} (1-\lambda)^{n-j} Z_{j,n} \right)}},$$
where
$$Z_{j,n} = \frac{1}{m+n} \max\left( 0, R_{j,n} - \frac{m+n+1}{2} \right),$$
and $R_{j,n}$ denotes the rank of $Y_j$ in the pooled sample $S_n=\{ U_{-m+1},..., U_{0},V_1...,V_n\}$.

We propose choosing the control limit $c_{n}(\alpha)$ of $T_{n,w,\lambda}$, similar to \cite{chen2016distribution}, to keep the false alarm rate of the control chart at a desired level, defined by
\begin{equation}\label{eq:UDFM_eqs}
        \mathbb{P}\bigg(T_{n,w,\lambda}\geq  c_{n}(\alpha) \mid S_n,  T_{k,w,\lambda} \leq c_{k}(\alpha) \text{ for } 1 \leq k \leq n-1\bigg) = \alpha.
\end{equation}
We then define the run length (RL) of the control chart as:
$RL = \min\{n : T_{n,w,\lambda} > c_{n}(\alpha), \, n \geq 1\}.$
\begin{proposition}
Under $H_0$, the run length follows a geometric distribution:
$\mathbb{P}(RL = n) = \alpha(1 - \alpha)^{n - 1}, \quad \text{for all } n \geq 1.$
\end{proposition}
\begin{proof}
    The proof follows that of Theorem 1 in \cite{chen2016distribution} and is omitted.
\end{proof}

\begin{proposition}\label{prop:conditional-equality}
$\mathbb{P}(T_{n,w,\lambda} \leq c_{n}(\alpha) \mid S_n,T_{k,w,\lambda} \leq c_{k}(\alpha) \text{ for } 1 \leq k \leq n-1) \; = $
\\
$\mathbb{P}(T_{n,w,\lambda} \leq c_{n}(\alpha) \mid S_n,T_{k,w,\lambda} \leq c_{k}(\alpha) \text{ for } n-w \leq k \leq n-1) 
$
\end{proposition}
\begin{proof}
    The proof follows from that of Proposition 3 in \cite{chen2016distribution}, since $Z_{j,n}$ depends only on $R_{j,n}$.
\end{proof}
By Proposition~\ref{prop:conditional-equality}, it suffices to consider only the last $w$ charting statistics when approximating $c_{n}(\alpha)$. This results in a reduced computational in (\ref{eq:UDFM_eqs}). The moments required to construct the UDFM control chart are given in Appendix \ref{AppUDFmoments} (supp. materials).

\subsection{Detection Based on Manifold Fitting Framework}

Assume $m$ Phase I observations are  collected. These observations define a region $\Gamma=\{y \in {\mathbb{R}^D}|\text{dist}(y,\mathcal{M})<C\sigma\}$ from which a manifold $\hat{\mathcal{M}}$ can be fitted, assuming that $\sigma$ is sufficiently small and $\operatorname{rch}(\mathcal{M})$ is large. To do this, when the true value of $\sigma$ is known, Algorithms ~\ref{alg:contraction_direction} and~\ref{alg:contract_point} can be initialized directly with appropriate constants, ensuring that each Euclidean ball and hypercylinder are adequately populated. In cases where $\sigma$ is unknown, it can be estimated first using Algorithm~\ref{alg:estimate_sigma_residual}.

Next, $m_1$ in-control (IC) deviations from the fitted manifold are estimated. Since these might be autocorrelated, we fit a (univariate) autoregressive (AR) model to them. After fitting the AR model, filtered values of $m_2$ IC deviations are passed into the UDFM control chart calibrated to a desired false alarm rate $\alpha$. At least 10 IC deviations and a window size $w \geq 5$ is recommended, as smaller values may prevent the approximated control limit $c_{n}(\alpha)$ from achieving the desired false alarm rate $\alpha$.

Once the control chart is set up, monitoring proceeds by sequentially collecting and observing new data, estimating the distance of each new observation to the manifold, filtering the estimated distance using the same AR model, and then inputting the resulting forecast residual into the UDFM control chart. This process is repeated until an alarm is signaled by the UDFM control chart. The algorithmic formulation of the monitoring procedure is given below.
\begingroup
\spacingset{1.2}
\begin{algorithm}[htbp!]
\caption{Manifold fitting Phase II SPC Framework}
\label{alg:mf_spc}
\begin{algorithmic}[1]
\REQUIRE $m$ Phase I observations $\{Y_t \in \mathbb{R}^D\}_{t=-m+1}^0$, constants $C_0,C_1,C_2$, false alarm rate $\alpha$, window size $w$

\STATE \textbf{Phase I: Manifold Fitting}
\STATE Fit $\hat{\mathcal{M}} \subset \mathbb{R}^D$ with the Phase I data.
\STATE Compute $m_1$ estimated IC deviations and fit an AR model.
\STATE Compute $m_2$ estimated IC deviations and filter them via the AR model.
\STATE Feed filtered quantities into the UDFM control chart with false alarm rate $\alpha$.

\STATE \textbf{Phase II: Online Monitoring}
\WHILE{no alarm signaled by the UDFM control chart}
    \STATE Collect and observe a new data point.
    \STATE Estimate the deviation of the new observation from $\hat{\mathcal{M}}$.
    \STATE Filter the estimated deviation using the same AR model.
    \STATE Input the resulting forecast residual into the UDFM control chart.
\ENDWHILE
\end{algorithmic}
\end{algorithm}
\endgroup

\section{Manifold Learning Based SPC}\label{manifold_learning}

In this section, we propose an alternative framework for the SPC problem defined in (\ref{prob_definition}). This method monitors the lower-dimensional representation of the process by approximating an embedding function $\hat{f}$ using Phase I observations, in line with prior work by \cite{li2015}, \cite{zhan2019improved}, \cite{cui2022nonparametric}, and \cite{ma2015fault}. Unlike earlier approaches, we filter temporal dependencies in the lower-dimensional space, which allows us to design a monitoring scheme with controllable $\text{ARL}_{in}$. We begin by introducing key concepts of this framework.

Suppose that there is a twice-differentiable function $h: \mathcal{M} \to \mathbb{R}$ defined on $\mathcal{M}$. For any nearby points $x, z \in \mathcal{M}$, \cite{belkin2001laplacian} showed that:
$$|h(x) - h(z)| \approx |\langle \nabla h(x), ||z - x|| \rangle| \leq \|\nabla h(x)\| \cdot \|x - z\|$$
This implies that the difference in function values is locally controlled by the gradient norm $\|\nabla h(x)\|$. Therefore, in order to preserve locality (meaning that nearby points on the manifold remain close after mapping to $\mathbb{R}$), one can seek a function that minimizes the total variation of $h$. This leads to the objective given in \cite{belkin2001laplacian}:
$$\arg\min_{\|h\|_{L^2(\mathcal{M})} = 1} \int_{\mathcal{M}} \|\nabla h\|^2$$
which is equivalent to
$$\int_{\mathcal{M}} \|\nabla h\|^2 = \int_{\mathcal{M}} \mathcal{L}(h)h$$
where $\mathcal{L}$ denotes the Laplace Beltrami operator acting on smooth functions defined on the manifold $\mathcal{M}$, i.e., $\mathcal{L}(h) = -\operatorname{div}(\nabla h)$. Thus, solving the locality preserving optimization problem above reduces to finding the eigenfunctions of the Laplace Beltrami operator on $\mathcal{M}$, which define an embedding \citep{perrault2011metric}.

This formulation establishes the foundation for Laplacian-based dimensionality reduction methods, including LPP and NPE. These Laplacian-based methods can be viewed as discrete approximations to the Laplace–Beltrami operator, constructed from finite sample data. Specifically, given a set of observations assumed to lie on or near a smooth manifold $\mathcal{M}$, one can construct an embedding into a lower-dimensional space $\mathbb{R}^d$. Not all Laplacian-based dimensionality reduction methods yield an explicit embedding function $\hat{f} : \mathcal{M} \to \mathbb{R}^d$, which poses a challenge for their use in online monitoring. However, if the form of $\hat{f}$ is restricted to be linear (as in LPP and NPE), it can be recovered and applied to new samples during on-line monitoring. It is important to note that a linear $\hat{f}$ does not imply these methods are linear manifold learning methods such as PCA. As shown by \cite{lpp}, the linear projections found via LPP are optimal linear mappings preserving the locality on a manifold, i.e., optimal linear approximations to the eigenfunctions of the Laplace-Beltrami operator $\mathcal{L}$. \cite{npe} show how NPE approximates the iterated graph Laplacian $\text{L}^2$ (L is regarded as a discrete approximation to $\mathcal{L}$ and converges to $\mathcal{L}$ under certain assumptions \citep{belkin2008towards}). In this sense, NPE finds a discrete approximation to the eigenfunctions of the iterated Laplacian $\mathcal{L}^2$. Therefore, both methods are capable of discovering the nonlinear manifold structure. More details regarding LPP and NPE can be found in Appendix \ref{AppManLearn} (supp. materials).

Suppose that we observe $m$ Phase~I samples $\{Y_t\}_{t = -m+1}^{0}$ in $\mathbb{R}^D$ and approximate an embedding function $\hat{f}$ using methods such as PCA, LPP or NPE. This allows us to represent $\{Y_t\}$ in $\mathbb{R}^d$ as
$y_t = \hat{f} (Y_t), \quad t = -m+1,-m+2, \dots,$ where $\{y_t\}$ is a stationary process. Thus, $\{Y_t\}$ reduces to the lower-dimensional process $\{y_t\}$. However, the lower-dimensional process still requires monitoring. One approach is to fit a time-series model, such as a Vector Autoregressive regression (VAR) in $\mathbb{R}^d$ (see \cite{papaioannou2022time}), and monitor forecast residuals online, but this typically requires many observations for parameter estimation. As a simpler alternative, we fit independent AR models to each dimension to filter autocorrelations. The filtered residuals approximate i.i.d. behavior, making the multivariate DFEWMA chart  \citep{chen2016distribution} suitable for monitoring $\{y_t\}$ (our findings in Section \ref{sec:toy_example} and ~\ref{tep_results} validate this practical assumption). 


\subsection{Detection Based on Manifold Learning Framework}


Assume as before that $m$ Phase I observations are collected. Both NPE and LPP yield an embedding function $\hat{f}$, obtained by solving a generalized eigenvalue problem, which becomes ill-posed when $m < D$, so in this method, {\em contrary to our manifold fitting method}, it is necessary to that $m > D$. Once $\hat{f}$ is obtained, $m_1$ observations are embedded into $\mathbb{R}^d$ and used to fit univariate AR models for each dimension. After the AR models are fitted, $m_2$ observations are embedded into $\mathbb{R}^d$, and their residuals are computed. These residuals are then used to construct the DFEWMA control chart for given false alarm rate $\alpha$,  weighting parameter $\lambda$, and window size $w$ \citep{chen2016distribution}.

Once the control chart is set up, monitoring proceeds by sequentially observing new data, embedding each new observation into $\mathbb{R}^d$ via $\hat{f}$, filtering the embedded data using the same AR models, and then inputting the resulting forecast residuals into the DFEWMA control chart. This process is repeated until an alarm is signaled by the DFEWMA control chart. The algorithmic formulation of the monitoring procedure is given below.
\begingroup
\spacingset{1.2}
\begin{algorithm}[htbp!]
\caption{Manifold Learning Phase II SPC Framework}
\label{alg:ml_spc}
\begin{algorithmic}[1]
\REQUIRE $m$ Phase I observations $\{Y_t\}_{t=-m+1}^0$, the intrinsic dimensionality $d$, false alarm rate $\alpha$, window size $w$

\STATE \textbf{Phase I: Manifold Learning}
\STATE Compute $\hat{f}$ by using $m$ Phase I observations.
\STATE Embed next $m_1$ observations into $\mathbb{R}^d$ by $\hat{f}$.
\STATE Fit AR models with the embedded $m_1$ observations
\STATE Embed next $m_2$ observations by $\hat{f}$ and compute their forecast residuals.
\STATE Set the DFEWMA control chart with the forecast residuals.

\STATE \textbf{Phase II: Online Monitoring}
\WHILE{no alarm signaled by UDFM control chart}
    \STATE Collect and observe a new data point.
    \STATE Embed new data point into $\mathbb{R}^d$ by $\hat{f}$.
    \STATE Compute forecast residuals using the same AR models.
    \STATE Input the resulting forecast residual into the DFEWMA.
\ENDWHILE
\end{algorithmic}
\end{algorithm}
\endgroup

\section{Performance Analysis}\label{performance_analysis}


\subsection{Synthetic Process on a 2-sphere  manifold}\label{sec:toy_example}

We generated a synthetic process as defined in (\ref{eq:synthetic_latent}) and (\ref{eq:synthetic_ambient}) with latent dimension $d=2$ and ambient dimension $D=6$. The parameters were set to $\sigma = 0.1$ and $\sigma_x = 0.3$, and $1200$ Phase I observations were simulated. After the $1200^{\text{th}}$ observation, we introduced a mean shift $\delta \sigma \in \mathbb{R}$, applied either to the first or fourth dimension of the ambient space. Note that $\mathcal{S}^2$ lies entirely within the first three dimensions of the ambient space, and we set $d = 3$ for the manifold learning algorithms.  Consequently, any faults introduced in the last 3 dimensions remain undetected with a manifold learning method for dimensionality reduction. To address this limitation, SPC monitoring procedures based on dimensionality reduction are typically supplemented with an additional control chart that tracks reconstruction errors. But then the two-charts  needs to be tuned to ensure an overall in-control ARL, and, in a sense, the two charts are monitoring the entire ambient space. 
In contrast, our proposed manifold fitting method avoids dimensionality reduction and instead monitors deviations from the fitted manifold directly in the ambient space with a {\em univariate} chart, eliminating the need for  two multivariate chart approach. 
\begingroup
\spacingset{1.2}
\begin{table}[htbp!]
\centering
\begin{tabular}{|c|c|c|c|c|c|}
\hline
\textbf{Fault Dimension} & $\delta$ & \textbf{MF} & \textbf{NPE} & \textbf{LPP} & \textbf{PCA} \\
\hline
-- & 0 & 20.92(21.05) & 19.73(19.59) & 19.91(20.03) & 19.79.41(19.82) \\
\hline
1 & 3 & 7.17(7.09) & 17.44(18.22) & 17.78(19.07) & 17.74(18.75) \\
\hline
1 & 10 & 2.74(1.85) & 8.94(11.16) & 9.05(11.74) & 9.15(11.86) \\
\hline
4 & 3 & 2.67(1.31) & 19.73(19.72) & 19.85(20.01) & 19.76(19.84) \\
\hline
4 & 10 & 1.99(0.67) & 19.26(19.49) & 19.75(19.93) & 19.72(19.73) \\
\hline
\end{tabular}
\caption{ARL values of the methods for various faults along different dimensions in the 2-sphere manifold example ($D=6$), with SDRL shown in parentheses. Each ARL is computed from 10,000 simulations. MF refers to the SPC procedure based on manifold fitting. In each simulation, the first $700$ Phase I observations were used to fit or learn the manifold, the next $400$ were used to fit AR(10) models to each dimension, and the final $100$ were used to set the control charts, following the procedures in Algorithms \ref{alg:mf_spc} and \ref{alg:ml_spc}. For manifold learning (LPP, NPE and PCA), we embedded into $\mathbb{R}^3$ and constructed graphs using 15 nearest neighbors. For manifold fitting, $\sigma$ was estimated and the parameters were fixed at $C_0=C_2=5$ and $C_1=3$. }
\label{tab:synthetic_results_hd}
\end{table}
\endgroup

Table \ref{tab:synthetic_results_hd} presents the ARL values for the different approaches in the synthetic process on a 2-sphere. For each method, ARL$_{\text{in}}$ is very close to the nominal target value of 20. As expected, the mean shift applied to the fourth dimension cannot be detected when monitoring only the lower-dimensional space obtained by the manifold learning methods.  In contrast, the SPC procedure based on manifold fitting (MF) not only detects all the considered shifts but also does so faster than the manifold learning–based SPC.

\subsection{A Replicated Tennessee Eastman (TE) Process}\label{tep_results}

The Tennessee Eastman Process Simulator is a continuous chemical process initially proposed in \cite{DOWNS1993245} as a benchmark for control engineering.
It is an open-loop system without feedback controllers, making it inherently unstable. 
We use instead the revised TE simulator of \cite{BATHELT2015309}, which includes feedback control loops. This TE simulator is shown in Figure \ref{tep}. 
As described in Appendix \ref{AppTE} (supp. materials) we further modified this TE simulator
and used 10 parallel copies of it to have available 300 variables.

We simulated a single realization of 10 TE simulators over 1,200 observations. The first 700 observations were used to estimate the intrinsic dimensionality $d$ with three methods from the \texttt{skdim} Python package \citep{bac2021scikit}, 
yielding estimates between 20 and 23. Figure~\ref{fig:pca_tep} presents the standard PCA scree plot and the cumulative explained variance. Based on these results, and supported by the PCA scree plot and cumulative explained variance in Figure~\ref{fig:pca_tep}, we set $d=22$. We then approximated $\hat{f}$ using NPE and LPP, both constructed with a 15-nearest neighbor graph. For LPP, the scale parameter of the heat kernel was set to the median of the pairwise distances \citep{lpp}.

\begingroup
\spacingset{1.2}
 \begin{figure}
 \centering
  \includegraphics[width=1\textwidth]{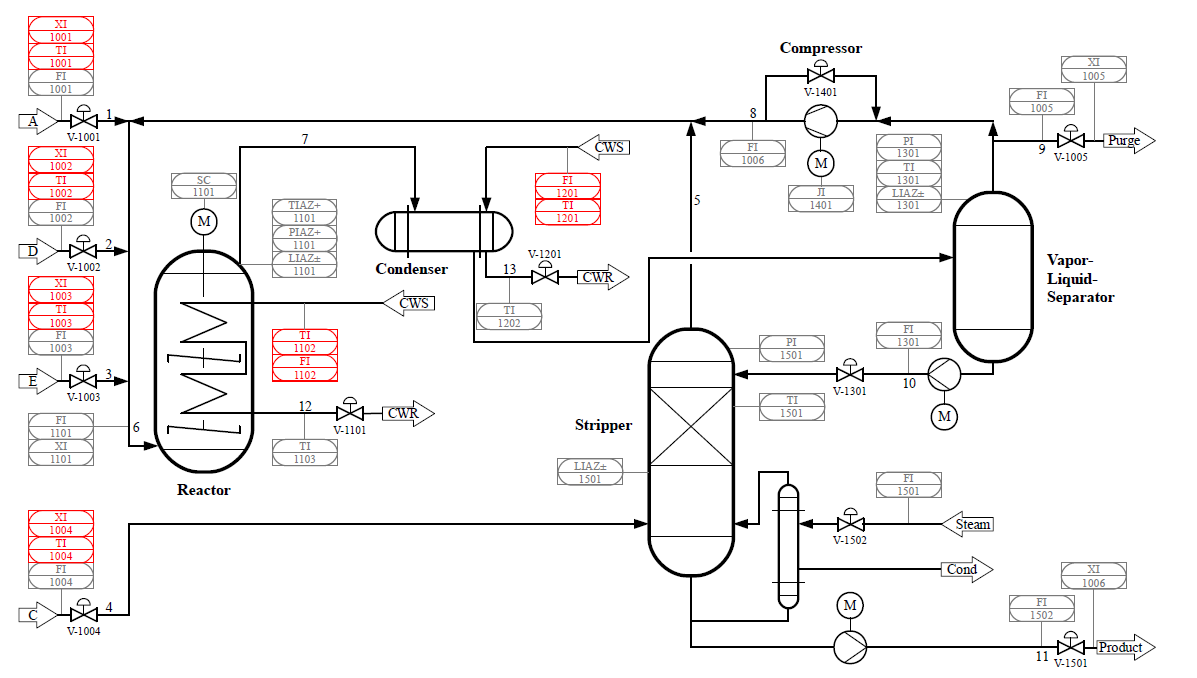}
 \caption{Tennesse Eastman Process from \cite{BATHELT2015309}. The process consists of five main unit operations: a reactor, a product condenser, a vapor-liquid separator, a recycle compressor, and a product stripper. Gaseous reactants A, C, D, and E are fed to produce two liquid products G and H. The feedback control loops are not illustrated here. Ten independent parallel TE simulators were used as described in Appendix D, supplementary materials.}
  \label{tep}
 \end{figure}
\endgroup

Next, we embedded the subsequent 400 observations into the lower-dimensional space $\mathbb{R}^{22}$ and fitted univariate AR(20) for each dimension independently. The remaining 100 observations were then embedded and filtered using the AR models. Figure~\ref{fig:p_values} reports the p-values from the Shapiro–Wilk normality test, applied separately to each dimension of the filtered quantities in the lower-dimensional space. The results indicate that forecast residuals do not follow a normal distribution for the nonlinear manifold learning methods, but are normal for PCA, as expected. We also applied the manifold fitting algorithm to the first 700 observations with parameters \(C_0 = C_2 = 22, C_1 = 14\), and estimated \(\sigma = 0.00037\) using Algorithm~\ref{alg:estimate_sigma_residual}.

\begingroup
\spacingset{1.2}
 \begin{figure}[htbp!]
    \centering
    \begin{subfigure}[b]{0.45\textwidth}
        \centering
        \includegraphics[width=\textwidth]{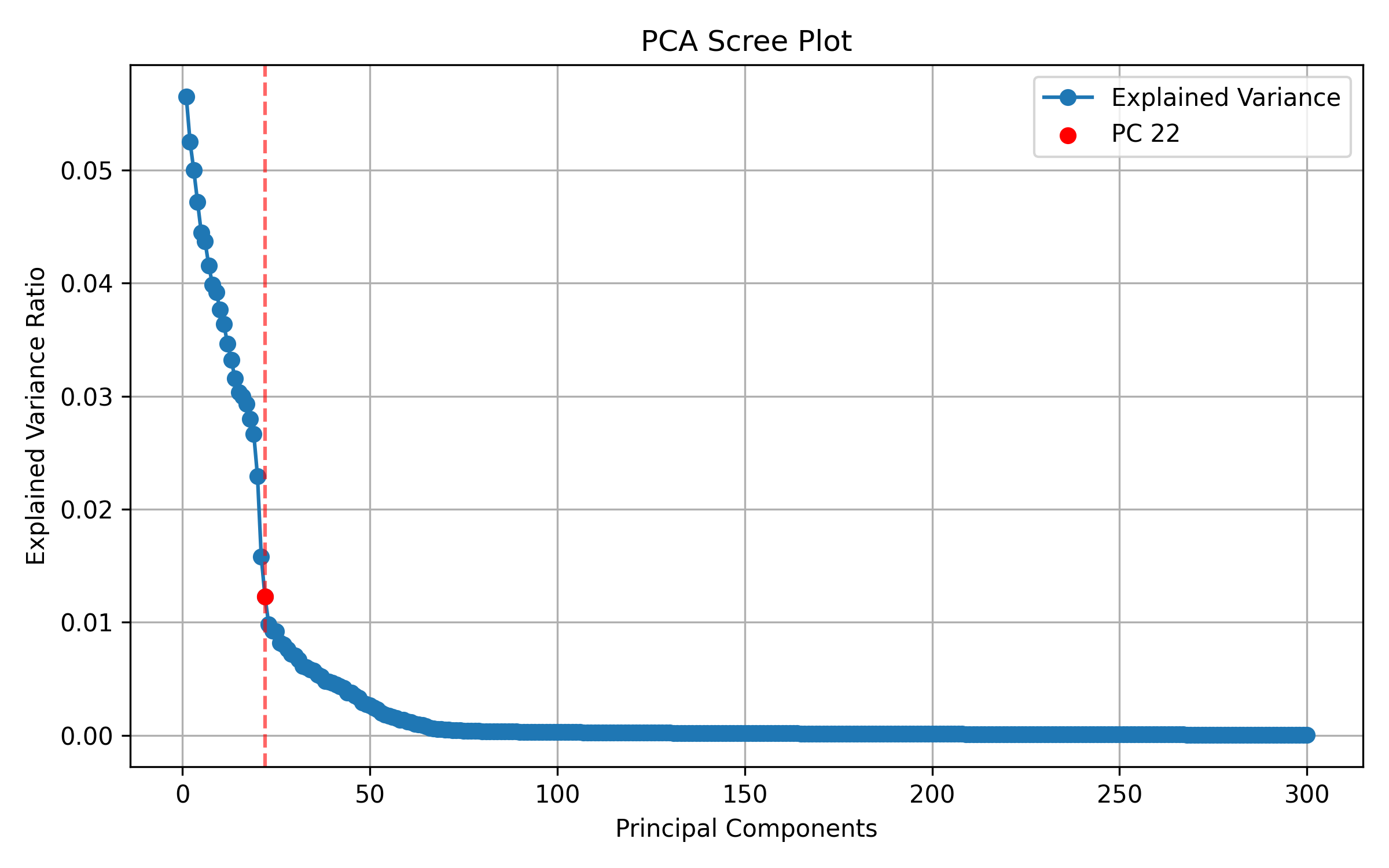}
        \caption{PCA Elbow Plot}
    \end{subfigure}
    \hfill
    \begin{subfigure}[b]{0.45\textwidth}
        \centering
        \includegraphics[width=\textwidth]{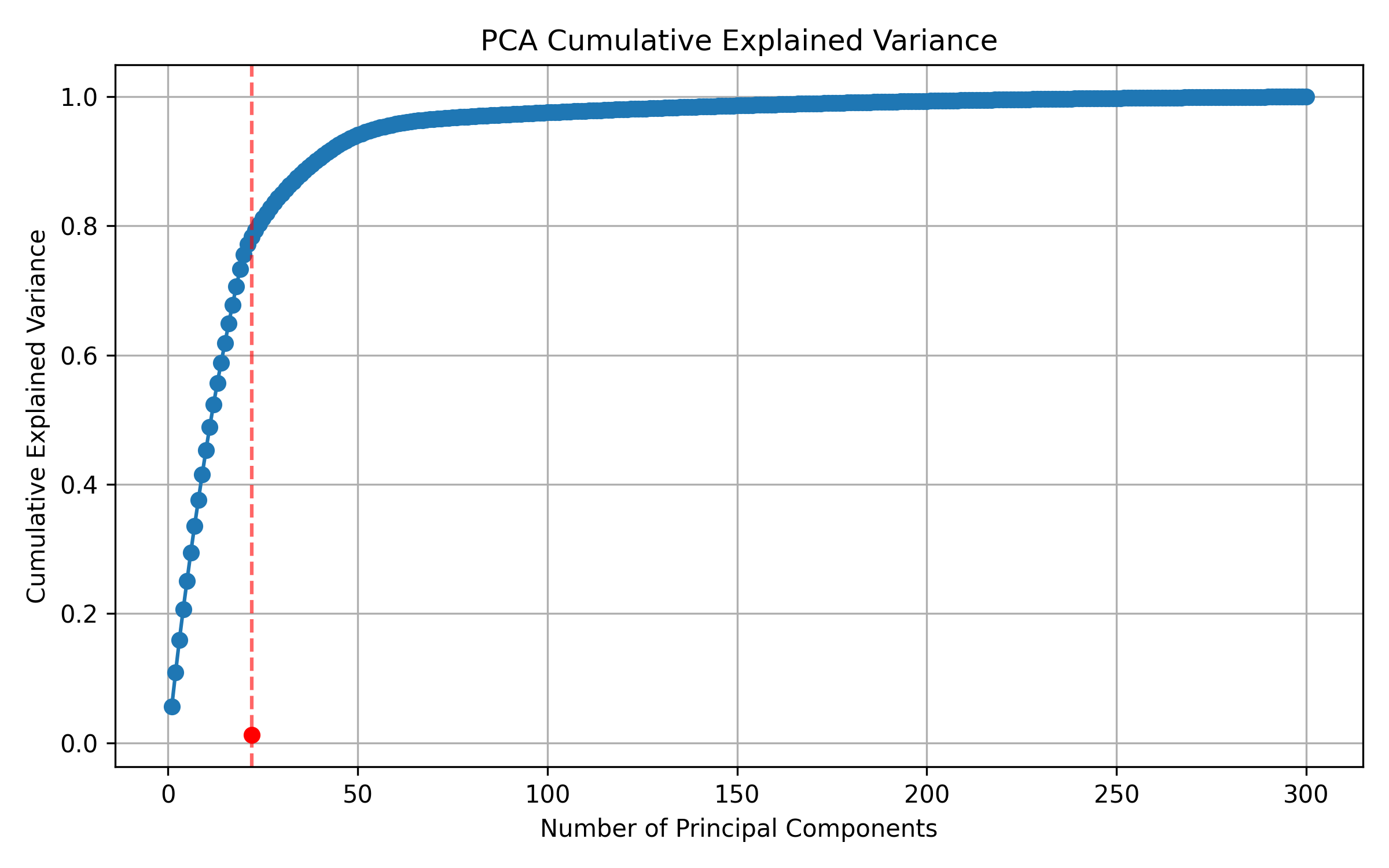}
        \caption{Cumulative Explained Variance}
    \end{subfigure}
    \hfill
    \caption{(a) A standard PCA elbow plot to estimate the intrinsic dimensionality. A noticeable drop occurs around the 22\textsuperscript{nd} principal component, which is highlighted by the red vertical line. (b) the cumulative explained variance across the principal components. The first 22 principal components account for approximately 80\% of the total variance.}
    \label{fig:pca_tep}
\end{figure}
\endgroup

\begingroup
\spacingset{1.2}
 \begin{figure}[htbp!]
    \centering
    \begin{subfigure}[b]{0.3\textwidth}
        \centering
        \includegraphics[width=\textwidth]{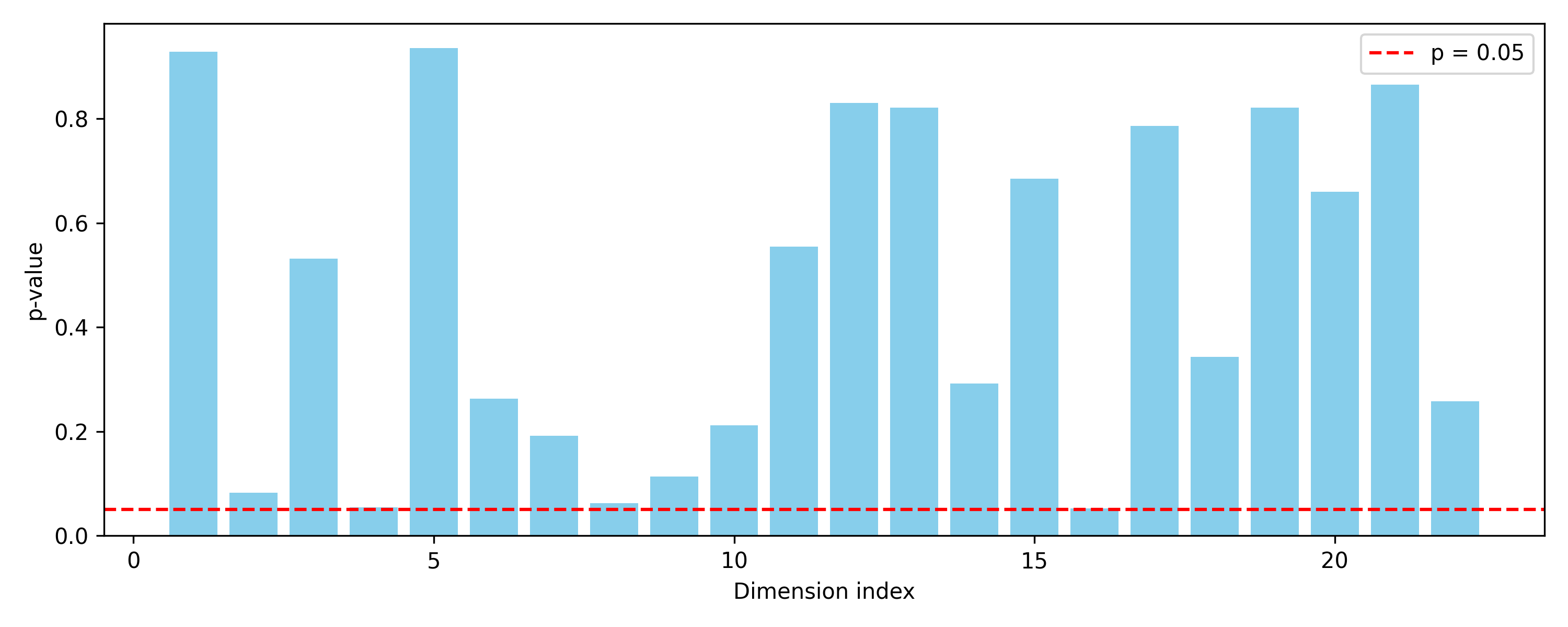}
        \caption{PCA}
    \end{subfigure}
    \hfill
    \begin{subfigure}[b]{0.3\textwidth}
        \centering
        \includegraphics[width=\textwidth]{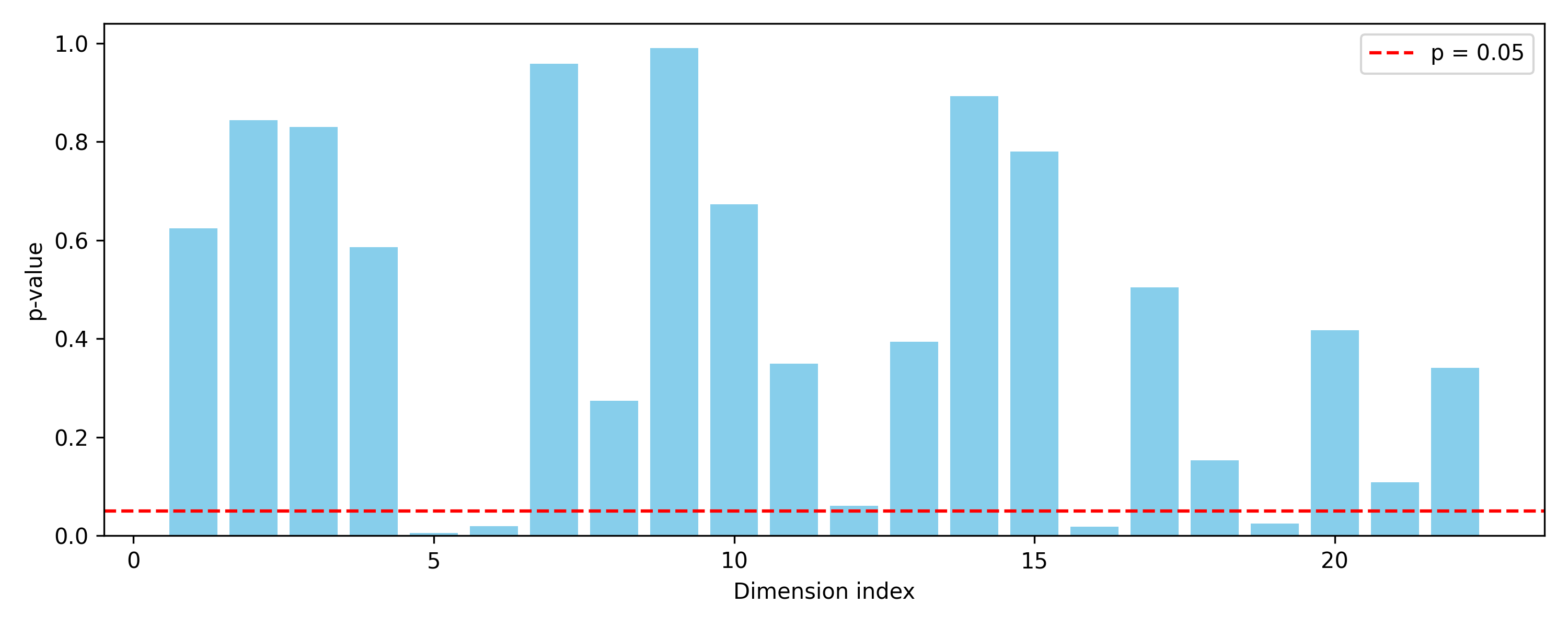}
        \caption{NPE}
    \end{subfigure}
    \hfill
    \begin{subfigure}[b]{0.3\textwidth}
        \centering
        \includegraphics[width=\textwidth]{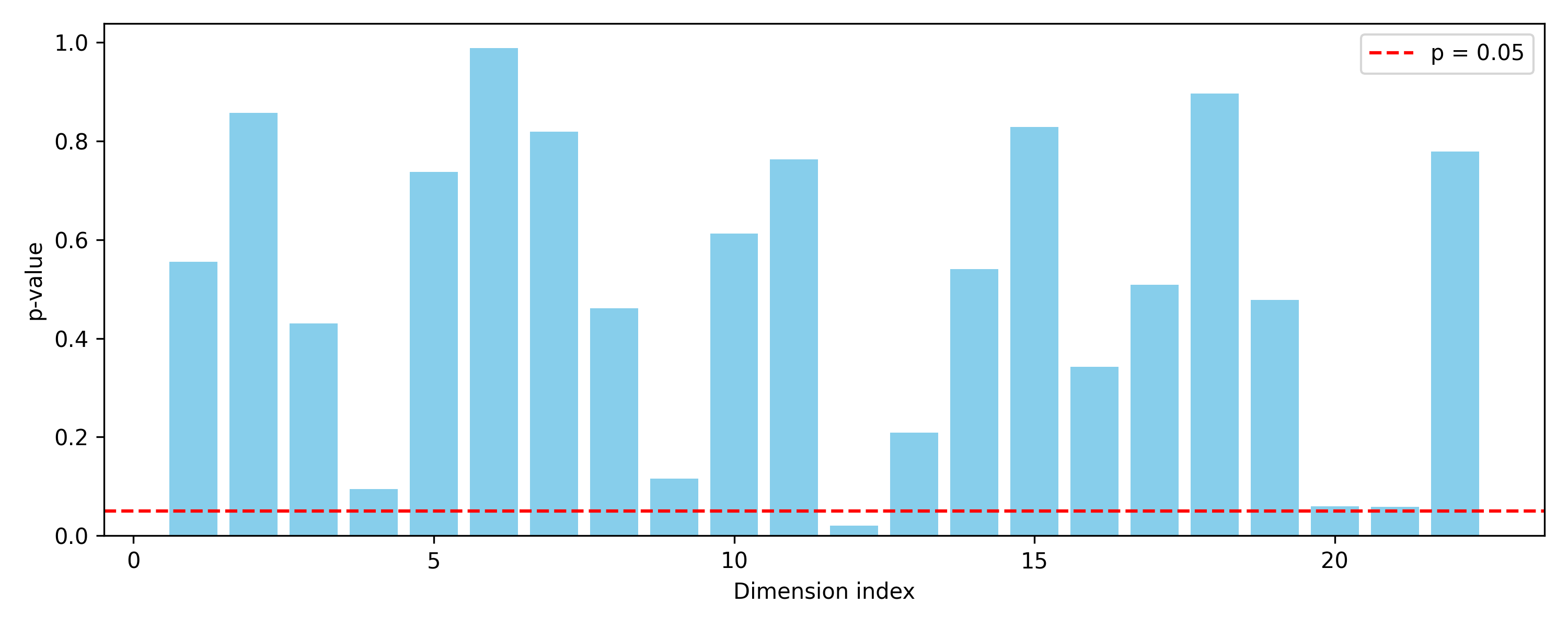}
        \caption{LPP}
    \end{subfigure}
    \hfill
    \caption{P-values from the Shapiro normality test, applied separately to each dimensionality reduction method.  While filtered quantities are relatively normally distributed when PCA is used, they deviate from normality when NPE and LPP are used. This shows the need for a distribution-free control chart method in the lower-dimensional space even if the ambient data are normal.}
    \label{fig:p_values}
\end{figure}
\endgroup
We evaluated the performance of both of the proposed SPC frameworks on the TE simulator faults~3, 4, and~9. Fault~4 affects the cooling-water inlet temperature of the reactor, while Faults~3 and~9 impact the temperature of stream~2 \citep{BATHELT2015309}. Our analysis focuses in particular on Faults~3 and~9, which are regarded as ``undetectable'' in the Chemometrics literature \citep{ma2015fault}. Each fault was introduced into only one of the ten TE simulators, so that the fault affected a single simulator, thereby increasing the difficulty of detection.

Unlike previous work, we simulated the faults at lower amplitudes, ranging from 0.05 to 0.1. Table~\ref{tab:tep_results} presents the ARL values for the proposed frameworks across different faults and amplitudes for the setting $m < D$. Each ARL$_{in}$ value is close to the nominal value 20, validating our practical assumption on the i.i.d. behavior of the forecast residuals. As noted previously, PCA is expected to perform well for linear manifolds. However, since the TE process is governed by various feedback control loops, the process data likely reside on a nonlinear manifold. As a result, the PCA-based monitoring approach failed to detect any of the simulated faults, unlike the other methods that account for a nonlinear structure. Among these, the manifold fitting (MF) approach outperformed both LPP and NPE in detecting faults 3 and 4 when the fault amplitude was set to 0.1, indicating that MF is more effective at identifying large shifts. For smaller shifts, the NPE-based approach achieved the best ARL$_{out}$, followed by the LPP-based approach. 
\begingroup
\spacingset{1.2}
\begin{table}[htbp!]
\centering
\begin{tabular}{|c|c|c|c|c|c|}
\hline
\textbf{Fault} & \textbf{Amplitude} & \textbf{MF} & \textbf{NPE} & \textbf{LPP} & \textbf{PCA} \\
\hline
-- & 0 & 20.95(23.15) & 19.22(18.92) & 19.46(19.39) & 20.41(20.41) \\
\hline
3 & 0.05 & 6.50(7.43) & 3.34(2.56) & 3.44(2.01) & 19.79(19.72) \\
\hline
3 & 0.1 & 2.09(0.72) & 2.20(0.66) & 2.30(0.70) & 19.93(19.46) \\
\hline
4 & 0.05 & 6.46(7.60) & 3.32(1.95) & 3.51(2.48) & 20.13(20.05) \\
\hline
4 & 0.1 & 2.09(0.72) & 2.21(0.67) & 2.32(0.71) & 19.91(19.84) \\
\hline
9 & 0.05 & 12.80(9.80) & 11.22(7.38) & 11.34(7.44) & 20.05(19.77) \\
\hline
9 & 0.1 & 10.67(6.97) & 9.84(5.89) & 9.94(5.92) & 19.80(19.71) \\
\hline
\end{tabular}
\caption{ARL values of the methods for various faults and amplitudes, TE simulations, with SDRL shown in parentheses. Each ARL value is computed from 10,000 simulations, with nominal $\text{ARL}_{\text{in}}=20$. The first $700$ Phase I observations were used to fit or learn the manifold, the next $400$ were used to fit AR(20) models, and the final $100$ were used to set the control charts, following the procedures in Algorithms \ref{alg:mf_spc} and \ref{alg:ml_spc}. For manifold learning (LPP, NPE and PCA), we embedded into $\mathbb{R}^{22}$ and constructed graphs using 15 nearest neighbors. For manifold fitting, $\sigma$ was estimated and the parameters were fixed at $C_0=C_2=22$ and $C_1=14$.}
\label{tab:tep_results}
\end{table}
\endgroup
\subsection{High dimensional performance}

We can only evaluate the performance of the MF approach in high-dimensional setting ($m < D$) given that as previously noted, both LPP and NPE become ill-posed in this case (as do PCA). 
Consequently, for $m=280 < D=300$ we only evaluated the MF approach. Table \ref{tab:tep_results_hd} presents the ARL values of the MF approach in this scenario. The computed ARL$_{in}$ remains close to 20, showing that the MF approach maintains an in-control ARL even in a high dimensional scenario. The fault detection power of the approach has declined compared to table \ref{tab:tep_results}, which used $m=700$, especially for faults 3 and 4 and amplitude 0.05. 
\begingroup
\spacingset{1.2}
\begin{table}[htbp!]
\centering
\begin{tabular}{|c|c|c|}
\hline
\textbf{Fault} & \textbf{Amplitude} & \textbf{MF}\\
\hline
-- & 0 & 21.33(24.85) \\
\hline
3 & 0.05 & 9.89(13.00) \\
\hline
3 & 0.1 & 2.65(2.79) \\
\hline
4 & 0.05 & 9.96(14.24)  \\
\hline
4 & 0.1 & 2.59(1.28) \\
\hline
9 & 0.05 & 13.89(12.06) \\
\hline
9 & 0.1 & 11.18(7.39) \\
\hline
\end{tabular}
\caption{ARL values of the MF approach for various faults and amplitudes in the extreme scenario where $m=280<D=300$, TE simulations, with SDRL shown in parentheses. Each ARL is computed from 10,000 simulations. In each simulation, the first 200 Phase I observations are used to fit a manifold, followed by 60 observations used to fit a univariate AR(p) model, with $p$ selected by the AIC criterion ($p\leq10$). The last 20 Phase I observations are used to set up the UDFM.}
\label{tab:tep_results_hd}
\end{table}
\endgroup

\subsection{The Kolektor surface-defect dataset}

The Kolektor surface-defect dataset, which is images of electrical commutators with surface defects is created by Kolektor Group d.o.o. and used for anomaly detection in a segmentation-based deep-learning framework by \cite{tabernik2020segmentation}. The dataset features images of 50 electrical commutators, with each commutator’s surface divided into eight distinct, non-overlapping sections. High-resolution images (1408 × 512 pixels) were captured for each section, resulting in a total of 399 images. Every commutator contains at least one surface defect, leading to 52 defective surface images and 347 non-defective ones. Example images of both defective and non-defective surfaces are shown in Figure~\ref{fig:KolektorSDD_samples}.

\begingroup
\spacingset{1.2}
\begin{figure}[htbp!]
    \centering
    \begin{subfigure}[b]{0.18\textwidth}
        \centering
        \includegraphics[width=\textwidth]{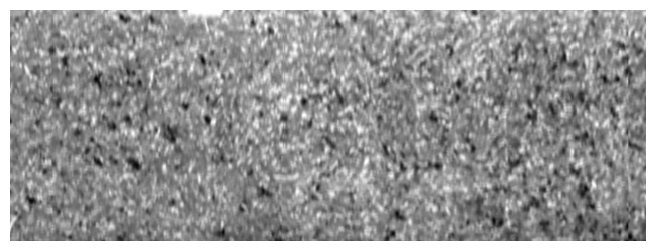}
    \end{subfigure}
    \hspace{0.5mm}
    \begin{subfigure}[b]{0.18\textwidth}
        \centering
        \includegraphics[width=\textwidth]{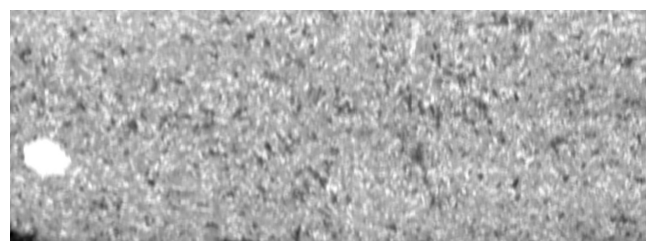}
    \end{subfigure}
    \hspace{0.5mm}
    \begin{subfigure}[b]{0.18\textwidth}
        \centering
        \includegraphics[width=\textwidth]{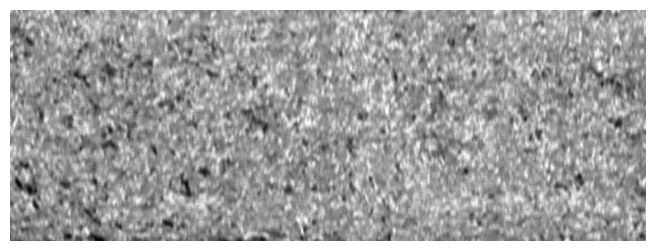}
    \end{subfigure}
    \hspace{0.5mm}
    \begin{subfigure}[b]{0.18\textwidth}
        \centering
        \includegraphics[width=\textwidth]{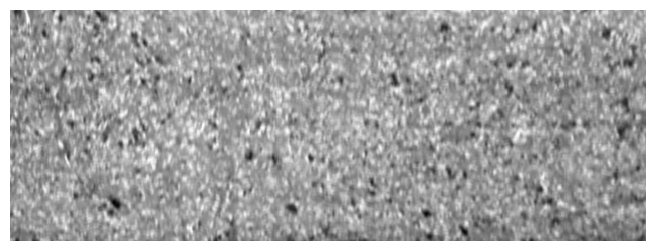}
    \end{subfigure}
    \hspace{0.5mm}
    \begin{subfigure}[b]{0.18\textwidth}
        \centering
        \includegraphics[width=\textwidth]{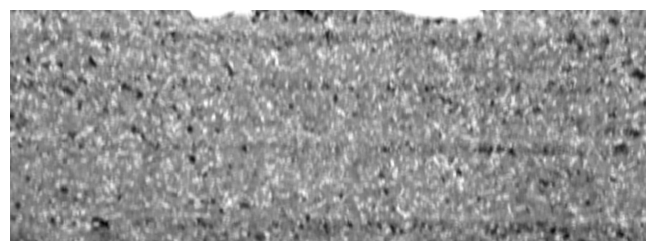}
    \end{subfigure}
    \centering
    \begin{subfigure}[b]{0.18\textwidth}
        \centering
        \includegraphics[width=\textwidth]{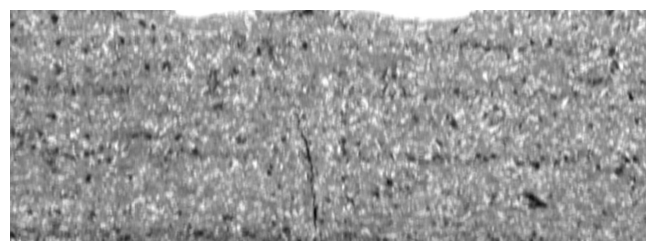}
    \end{subfigure}
    \hspace{0.5mm}
    \begin{subfigure}[b]{0.18\textwidth}
        \centering
        \includegraphics[width=\textwidth]{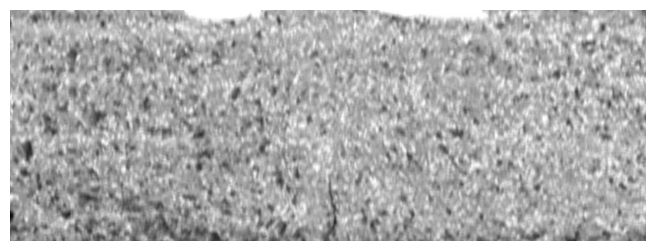}
    \end{subfigure}
    \hspace{0.5mm}
    \begin{subfigure}[b]{0.18\textwidth}
        \centering
        \includegraphics[width=\textwidth]{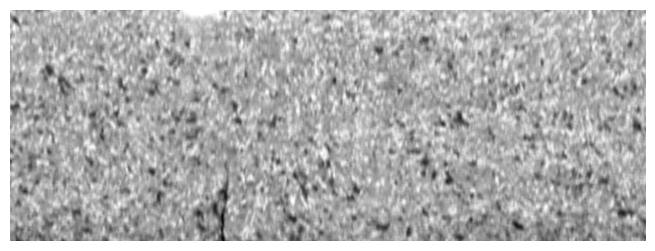}
    \end{subfigure}
    \hspace{0.5mm}
    \begin{subfigure}[b]{0.18\textwidth}
        \centering
        \includegraphics[width=\textwidth]{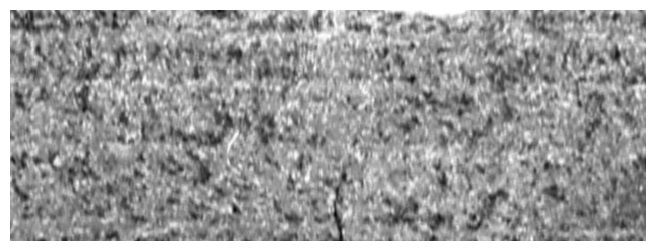}
    \end{subfigure}
    \hspace{0.5mm}
    \begin{subfigure}[b]{0.18\textwidth}
        \centering
        \includegraphics[width=\textwidth]{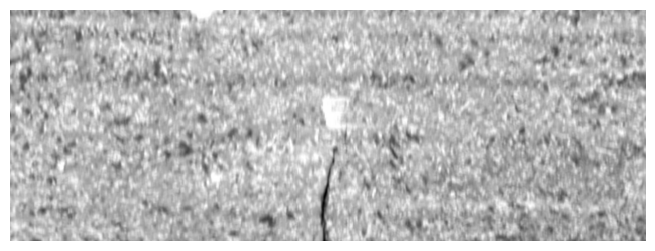}
    \end{subfigure}
    \caption{Example images from the Kolektor surface-defect dataset. The top row displays non-defective surfaces, while the bottom row shows defective ones. While some defects are clearly visible, others are subtle and difficult to see to the eye.}
    \label{fig:KolektorSDD_samples}
\end{figure}
\endgroup

We treated each image as a vector in $\mathbb{R}^{720896}$ and used 330 images of non-defective surfaces to fit a manifold. We set the tuning parameters to $C_0 = C_2 = 1100$ and $C_1 = 700$, and estimated $\sigma$ to be $0.08082$ in our MF method. The deviations of the remaining 17 non-defective images were then estimated to set up the UDFM control chart, after which the defective images were fed to the MF SPC algorithm. In this case, we did not apply any filtering to the estimated distances, as the images were assumed to be i.i.d. We set the nominal $\text{ARL}_{\text{in}}$ to 200, with $w = 5$ and $\lambda = 0.05$. In these settings, our MF SPC framework yielded a signal at $5^{th}$ observation. Figure~\ref{fig:kolektor_cc} displays the estimated distances and the test statistic of the control chart and its control limits.

\begingroup
\spacingset{1.2}
\begin{figure}[htbp!]
    \centering
    \begin{subfigure}[b]{0.47\textwidth}
        \centering
        \includegraphics[height=3cm]{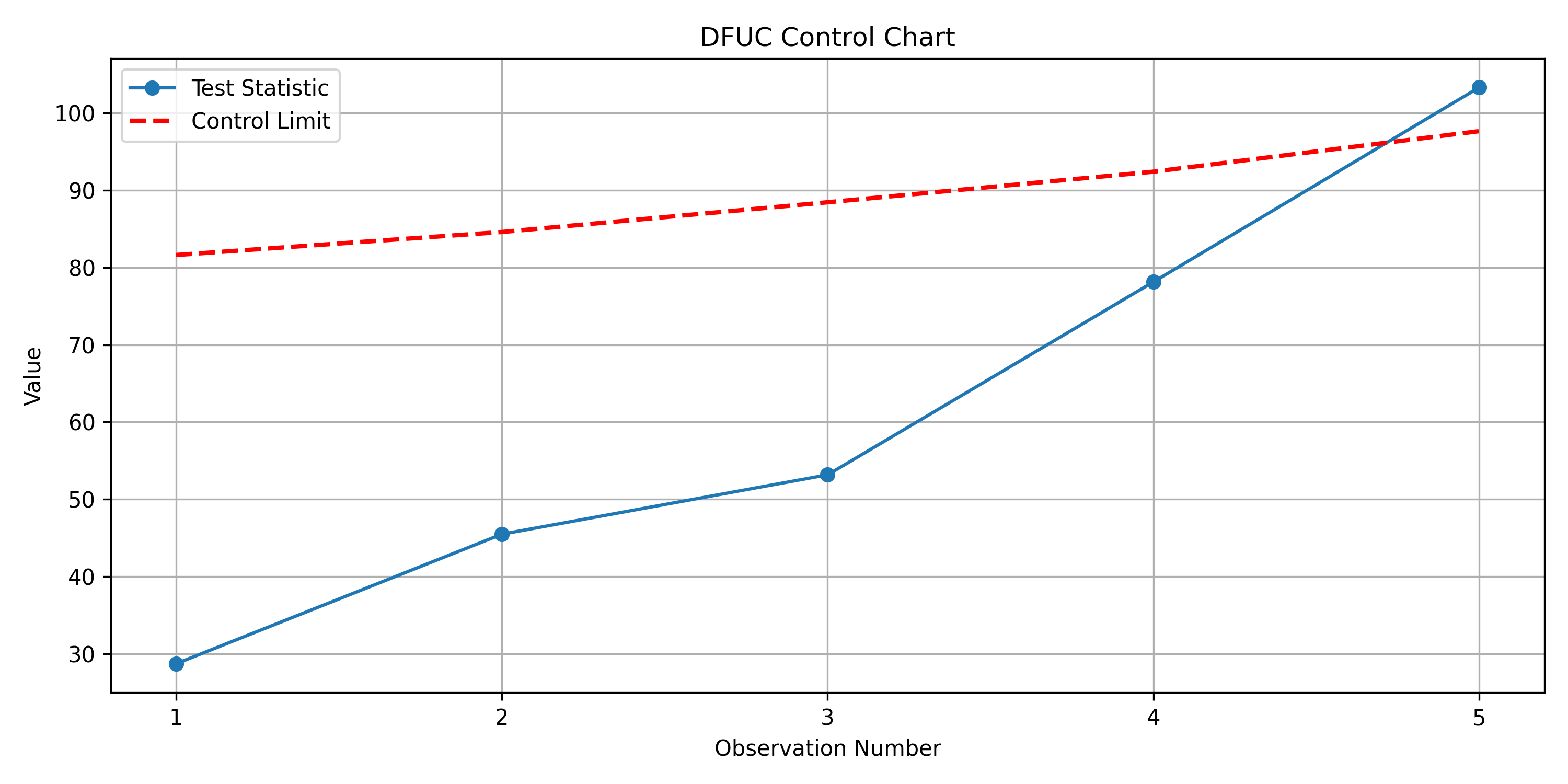}
        \caption{UDFM control chart}
    \end{subfigure}
    \hfill
    \begin{subfigure}[b]{0.47\textwidth}
        \centering
        \includegraphics[height=3cm]{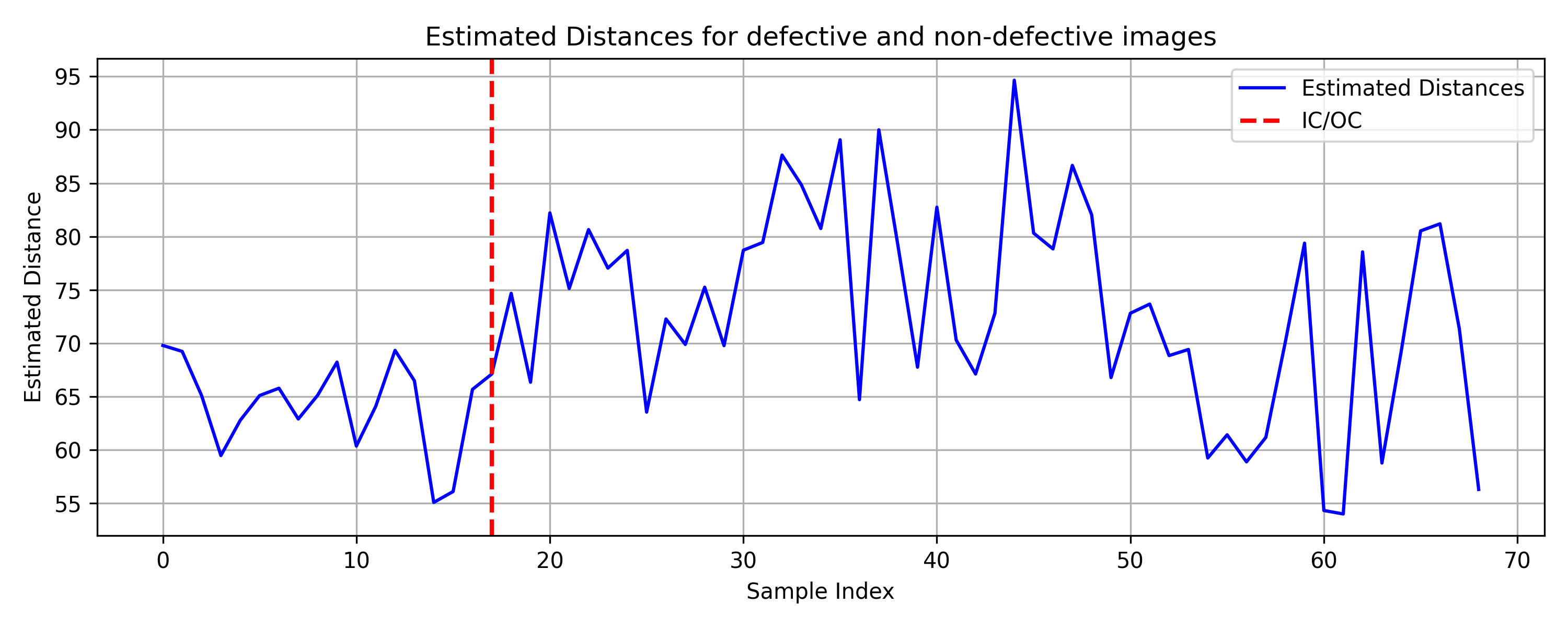}
        \caption{Estimated distances of non-defective and defective surfaces}
    \end{subfigure}
    \caption{ (a) The control chart raises an alarm at the $5^{\text{th}}$ defective surface image when $\text{ARL}_{\text{in}} = 200$. The red line represents the control limit, while the blue line indicates the test statistic. (b) The estimated distances for 17 non-defective and 52 defective surface images. The vertical red line separates the non-defective from the defective samples. The estimated distances for the defective surfaces are noticeably higher than those for non-defective ones.}
\label{fig:kolektor_cc}
\end{figure}
\endgroup
\section{Conclusion}\label{sec-conc}

Two distinct SPC frameworks for high-dimensional dynamic industrial processes, under the assumption that the process data lie on a nonlinear manifold were presented. The first framework leverages a state-of-the-art technique on manifold fitting to estimate a manifold in the ambient space from in-control data, such that deviations from the estimated manifold can be monitored. This remarkable simple idea results in a very effective SPC detection technique. The univariate sequence of deviations is monitored with a novel distribution-free univariate control chart with controllable  ARL$_{\text{in}}$. The second SPC framework adopts a more traditional approach: the high-dimensional observations are embedded into a lower-dimensional space and subsequently monitored with a multivariate control chart, also ensuring a controllable ARL$_{\text{in}}$. The manifold learning methods used have closed-form expressions for the out of sample extension, permitting on-line monitoring. In both methods, nonparametric charts are used given that the data to monitor is not normally distributed even if the original data in ambeint space were normal.

Extensive numerical simulations on a synthetic process on a d-sphere and on a parallel Tennessee Eastman Process demonstrate that the conceptually simpler manifold fitting based SPC framework achieves performance competitive with, and sometimes superior to, the more classical manifold learning based SPC methods. We also demonstrated the practical applicability
of the proposed manifold-fitting SPC framework by successfully detecting surface anomalies
in a real image dataset of electrical commutators.

Further work can be undertaken along the following lines. While both frameworks were developed to detect mean shifts in the ambient space, the proposed manifold-fitting approach has also a potential to detect shape changes in the underlying manifold, which would be analogous to detecting covariance changes in the process distribution. Moreover, considering the projected points onto the manifold rather than the deviations from it may reveal which variables caused the signal.


\section{Disclosure statement}\label{disclosure-statement}

The authors have no conflicts of interest to declare.

\section{Data Availability Statement}\label{data-availability-statement}

Code and data are available at \url{https://github.com/iburaktas/Manifold-SPC}, including examples of scaled-down versions of the experiments presented in this paper and a pipeline for applying the proposed SPC frameworks.

\spacingset{1.1}
\bibliography{bibliography.bib}
\newpage
\spacingset{1.6}

\phantomsection\label{supplementary-material}
\bigskip

\renewcommand{\thesection}{\Alph{section}}
\setcounter{section}{0}
{\large\bf SUPPLEMENTARY MATERIALS}

\section{Proof that the mean of deviations increases} \label{AppDeviationProof}
Recall our problem setup,
\begin{align*}
    Y_t =
\begin{cases}
X_t + \mathcal{E}_t, & \text{for }  t = -m + 1, \dots, \tau -1 \\
X_t + \Delta + \mathcal{E}_t, & \text{for } t = \tau,\tau +1, \dots
\end{cases}
\end{align*}
where $\mathcal{E}_t \overset{\text{i.i.d.}}{\sim} \mathcal{N}(0,\sigma^2I_D)$ and $X_t$ is supported on $\mathcal{M}$

Now, suppose $\mathcal{M}$ is $d$-dimensional manifold (closed, boundaryless) embedded in $\mathbb{R}^D$ with $\operatorname{rch}(\mathcal{M})=\infty$, then $\mathcal{M}$ is a $d$-dimensional affine subspace (see the proof of Theorem 4 in \cite{lieutier2024manifolds}). Without loss of generality, suppose $\mathcal{M}$ is a $d$-dimensional linear subspace. Let $P$ be the orthogonal projection operator on $\mathcal{M}$ and $Q = I_D - P$. Then the projection operator $\pi:\mathbb{R}^D\to \mathcal{M}$ satisfies
$$\pi(Y_t)=PY_t \quad \text{and} \quad Y_t-\pi(Y_t)=QY_t$$
Hence, for $t<\tau$ we have,
$$\|Y_t-\pi(Y_t)\|_2 =\|QY_t\|_2= \|QX_t+Q\mathcal{E}_t\|_2=\|Q\mathcal{E}_t\|_2$$
and for $t\geq\tau$ we have,
$$\|Y_t-\pi(Y_t)\|_2 =\|QY_t\|_2= \|QX_t+Q\Delta +Q\mathcal{E}_t\|_2=\|Q\Delta+Q\mathcal{E}_t\|_2$$
Note that $Q\mathcal{E}_t \sim \mathcal{N}(0,\sigma^2 I_{D-d})$ due to the isotropy of $\mathcal{E}_t$. Hence, $\|Y_t-\pi(Y_t)\|_2^2/\sigma^2 = \|Q\mathcal{E}_t\|_2^2/\sigma^2 \sim \chi^2_{D-d}$ for $t<\tau$ and $\|Y_t-\pi(Y_t)\|_2^2/\sigma^2 = \|Q\Delta + Q\mathcal{E}_t\|_2^2/\sigma^2 \sim \chi^2_{D-d}(\theta)$ with noncentrality parameter $\theta = \|Q\Delta\|_2^2/\sigma^2$ for $t\geq\tau$. 

Consequently, if $\Delta$ is not entirely tangent to $\mathcal{M}$, i.e., $\|Q\Delta\|_2^2 >0$, 
$$\mathbb{E}\left[\|\pi(Y_t) - Y_t\|_2^2 \mid t \geq \tau \right] =\sigma^2 (D-d) +\|Q\Delta\|_2^2  >\sigma^2 (D-d)= \mathbb{E}\left[\|\pi(Y_t) - Y_t\|_2^2 \mid t < \tau \right]$$

\section{A stochastic process on a sphere}
\label{AppStcProcess}
Let $P:\mathbb{R}^D \to \mathbb{R}^{D}$ denote the projection that retains only the first $d{+}1$ coordinates,
$$P(y) = (y_1,\dots,y_{d+1},0,\dots,0),$$
and let
$$\ker(P)=\{\,y\in\mathbb{R}^D : P(y)=0\,\}=\{(0,\dots,0,y_{d+2},\dots,y_D)\}.$$

Let $\mathcal{M}$ be a $d$-dimensional unit sphere $\mathcal{S}^d$ embedded in $\mathbb{R}^{D}$ for $d<D$, defined as $\mathcal{S}^d = \{ x \in \mathbb{R}^D : x_{d+2} = \cdots = x_D = 0, \; \sum_{i=1}^{d+1} x_i^2 = 1 \}$. The latent process $\{X_t\}$ evolves on $\mathcal{S}^d$ according to
$$X_t = \pi(X_{t-1} + E_{t}), \quad \text{for } t=-m+1,-m+2,\dots,$$
where $E_t \overset{\text{i.i.d.}}{\sim} \mathcal{N}(0, \sigma_x^2 I_{D})$, 
$X_{-m}\sim \text{Unif}(\mathcal{S}^d)$, 
and $\pi:\mathbb{R}^D \setminus \ker(P)\to \mathcal{S}^d$ is the projection operator defined by
$$\pi(y) = \frac{P(y)}{\|P(y)\|_2}.$$

Note that the distributions of $E_{-m+1}$ and $X_{-m}$ are rotation invariant. 
Specifically, for any $R \in \mathbb{R}^{D\times D}$ such that
$$R = \begin{bmatrix} 
M & 0 \\ 
0 & I_{D-d-1}
\end{bmatrix}, 
\qquad M \in SO(d{+}1),$$
$R E_{-m+1}\;\stackrel{d}{=}\;E_{-m+1}$ and  $R X_{-m}\;\stackrel{d}{=}\;X_{-m}$.

Hence, for any such $R$, we have
$$\pi(Ry) =\frac{P(Ry)}{\|P(Ry)\|_2} = \frac{P(Ry)}{\|P(y)\|_2} = \frac{RP(y)}{\|P(y)\|_2} = R\,\pi(y).$$

Consequently, for any measurable $A\subset\mathcal S^{d}$,
\begin{align*}
\mathbb{P}\!\left(X_{-m+1}\in A\right)
 &= \mathbb{P}\!\left(\pi(X_{-m}+E_{-m+1})\in A\right) \\
 &= \mathbb{P}\!\left(R\,\pi(X_{-m}+E_{-m+1})\in RA\right) \\
 &= \mathbb{P}\!\left(\pi(RX_{-m}+RE_{-m+1})\in RA\right) \\
 &= \mathbb{P}\!\left(\pi(X_{-m}+E_{-m+1})\in RA\right)\\
&= \mathbb{P}\!\left(X_{-m+1}\in RA\right).
\end{align*}

Since this holds for all $R$ as above, the law of $X_{-m+1}$ is the unique rotation-invariant probability measure on $\mathcal S^d$, i.e. $X_{-m+1}\sim\text{Unif}(\mathcal S^d)$.
Because one step preserves uniformity and the transition kernel is time-homogeneous, starting from 
$X_{-m}\sim\mathrm{Unif}(\mathcal S^d)$ implies $X_t\sim\mathrm{Unif}(\mathcal S^d)$ for every $t$. Hence, $\{X_t\}$ is strictly stationary.

Also, recall that the observed process $\{Y_t\}$ evolves according to
\begin{align*}
    Y_t = X_t + \mathcal{E}_t, \quad \text{for} \quad t= -m+1,-m+2,\dots,
\end{align*}
where \( \mathcal{E}_t \overset{\text{i.i.d.}}{\sim} \mathcal{N}(0, \sigma^2 I_D) \). Then for any R defined as above, $RY_t\;\stackrel{d}{=}\;Y_t$.

Now fix $x\in \mathcal{S}^d$, and choose $R$ of the above form such that $Rx=e_1$, where $e_1$ is a unit vector. Then, 
\begin{align*}
    \bigg(\|Y_t-\pi(Y_t)\| \mid X_{t}=x \bigg) & \;\stackrel{d}{=}\; \bigg(\|RY_t-\pi(RY_t)\| \mid X_{t}=x \bigg) \\ 
    & \;\stackrel{d}{=}\; \bigg(\|R{X_{t}}+R\mathcal{E}_t- \pi(R{X_{t}}+R\mathcal{E}_t)\| \mid  X_{t}=x \bigg) \\
    & \;\stackrel{d}{=}\; \bigg(\|e_1+R\mathcal{E}_t- \pi(e_1+R\mathcal{E}_t)\| \mid  X_{t}=x \bigg) \\
\end{align*}
Hence, the distribution of $\|Y_t-\pi(Y_t)\|$ depends only on $\mathcal{E}_t$. It follows that $\{\|Y_t-\pi(Y_t)\|\}$ are also i.i.d.

\section{Univariate Deviation From Manifold (UDFM) control chart}\label{AppUDFmoments}
Here, we present the moments of $Z_j$ under $H_0$, which are used in the UDFM.

\subsection[E(Zj | H0) = mu_N]{\(\mathbb{E}[Z_j \mid H_0] = \mu_N\)}

\begin{itemize}
    \item If \( N = 2p \) (even), then the ranks exceeding \( p \) contribute:
    \[
    \mathbb{E}[Z_j \mid H_0] = \frac{1}{N^2} \sum_{r=p+1}^{N} \left(r - \frac{N+1}{2} \right) = \frac{1}{N^2} \sum_{s=1}^{p} \left(s - \frac{1}{2} \right) = \frac{p^2}{2N^2} = \frac{1}{8}.
    \]

    \item If \( N = 2p + 1 \) (odd), then:
    \[
    \mathbb{E}[Z_j \mid H_0] = \frac{1}{N^2} \sum_{r=p+2}^{N} \left(r - (p + 1) \right) = \frac{1}{N^2} \sum_{s=1}^{p} s = \frac{p(p+1)}{2N^2} = \frac{N^2 - 1}{8N^2}.
    \]
\end{itemize}
Hence, $\lim_{N\to \infty}\mu_N=1/8$.

\subsection[Var(Zj | H0)]{\(\operatorname{Var}(Z_j \mid H_0)=\sigma^2_N\)}

\[
\mathbb{E}[Z_j^2 \mid H_0] = \frac{1}{N^2} \sum_{r=1}^{N} \left(\max\left(0, r - \frac{N+1}{2}\right)\right)^2 = \frac{N^2 - 1}{24N^2}
\]

\begin{itemize}
    \item If \( N = 2p \) (even):
    \[
    \operatorname{Var}(Z_j\mid H_0) = \frac{N^2 - 1}{24N^2} - \left(\frac{1}{8}\right)^2 = \frac{5N^2 - 8}{192N^2}
    \]

    \item If \( N = 2p+1 \) (odd):
    \[
    \operatorname{Var}(Z_j\mid H_0) = \frac{N^2 - 1}{24N^2} - \left(\frac{N^2 - 1}{8N^2}\right)^2 = \frac{(N^2 - 1)(5N^2 + 3)}{192N^4}
    \]
\end{itemize}
Hence, $\lim_{N\to \infty}\sigma^2_N=\sigma^2_Z=5/192$.

\subsection[Cov(Zj, Zi | H0) for j != i]{\(\operatorname{Cov}(Z_j, Z_i \mid H_0)\) for \(j \neq i\)}

Using the fact that:
\[
\operatorname{Var}\left(\sum_{j=1}^{N} Z_j\mid H_0\right) = 0 = N \operatorname{Var}(Z_1\mid H_0) + N(N-1) \operatorname{Cov}(Z_1, Z_2\mid H_0)
\]
we get
\[
\operatorname{Cov}(Z_j, Z_i\mid H_0) = -\frac{\operatorname{Var}(Z_1\mid H_0)}{N-1}=-\frac{\sigma_N^2}{N-1}
\]

\subsection[Var(sum_{j=1}^{n} (1-lambda)^{n-j} Zj | H0) for lambda in (0,1)]{$\operatorname{Var}\big(\sum_{j=1}^{n}(1-\lambda)^{\,n-j} Z_j \mid H_0\big)$ for $\lambda \in (0,1)$}

Let $a_j = (1 - \lambda)^{n - j}$. We define the following sums:
$$A = \sum_{j=1}^n a_j^2 = \sum_{j=1}^n (1 - \lambda)^{2(n - j)} = \sum_{k = 0}^{n - 1} (1 - \lambda)^{2k} = \frac{1 - (1 - \lambda)^{2n}}{\lambda(2 - \lambda)},$$
$$S = \sum_{j=1}^n a_j = \sum_{k = 0}^{n - 1} (1 - \lambda)^k = \frac{1 - (1 - \lambda)^n}{\lambda}.$$
Hence, 
$$
\sum_{1 \leq i < j \leq n} a_i a_j = \frac{1}{2} \left( S^2 - A \right).
$$

Now, 
\begin{align*}
    \operatorname{Var}\bigg(\sum_{j=1}^{n}(1-\lambda)^{n-j} Z_j\mid H_0\bigg) & = \operatorname{Var}\bigg(\sum_{j=1}^{n}a_j Z_j\mid H_0\bigg) \\
    & = \sum_{j=1}^n a_j^2 \operatorname{Var}(Z_j\mid H_0) + 2 \sum_{1 \leq i < j \leq n} a_i a_j \operatorname{Cov}(Z_i, Z_j\mid H_0) \\
    & = \sigma_N^2\sum_{j=1}^n a_j^2  - \frac{2\sigma_N^2}{N - 1}\sum_{1 \leq i < j \leq n} a_i a_j \\
    & =  \sigma_N^2A - \frac{\sigma_N^2}{N - 1} \left( S^2 - A \right) \\
    & = \sigma^2_N \left(1 + \frac{1}{N - 1} \right) A - \frac{\sigma^2_N}{N - 1} S^2 \\
    & =
\sigma^2_N \left(1 + \frac{1}{N - 1} \right)
\sum_{j=1}^n (1 - \lambda)^{2(n - j)}
- \frac{\sigma^2_N}{N - 1}
\left( \sum_{j=1}^n (1 - \lambda)^{n - j} \right)^2
\end{align*}

\section{Proof of Theorem 1}\label{AppTheorem1}
\begin{theorem}
Let $c_{m,n}(\alpha)$ be the $\alpha$ upper quantile of the distribution of $T^*_{m,n}$ given $S_{m,n}$. Suppose $m$ and $n$ are large enough that the exact $\alpha$ upper quantile is well defined. Then,
\begin{enumerate}[label=(\alph*)]
    \item Under $H_0$, $\mathbb{P}\bigg(T_{m,n} \geq c_{m,n}(\alpha) \bigg) = \alpha$.
    \item Under $H_1$, $\mathbb{P}\bigg(T_{m,n} \geq c_{m,n} (\alpha)\bigg) \to 1$ if $m,n \to \infty$ and $\frac{n}{N}\to \gamma \in (0,1)$
\end{enumerate}
\end{theorem}
\begin{proof} 
\begin{enumerate}[label=(\alph*)]
    \item Under $H_0$ and the given conditions,
    $$\mathbb{P}\bigg(T^*_{m,n} \geq c_{m,n}(\alpha) \bigg) = \mathbb{E}\bigg[\mathbb{P}\bigg(T^*_{m,n} \geq c_{m,n}(\alpha) \mid S_{m,n} \bigg)\bigg]= \alpha$$
    Note that $T_{m,n}$ has the same marginal distribution as $T^{*}_{m,n}$ under $H_0$, then,
    $$\mathbb{P}\bigg(T_{m,n} \geq c_{m,n}(\alpha) \bigg) = \alpha$$
    \item First, we will show that $c_{m,n}(\alpha)$ is upper bounded with a constant independent of $m,n$ under $H_0$. Note that $\mathbb{E}[T^*_{m,n}|S_{m,n}]=0$ and $\operatorname{Var}(T^*_{m,n}|S_{m,n})=1$ under $H_0$. Then, by Cantelli's inequality
        $$\mathbb{P}( T^*_{m,n} \geq c \mid S_{m,n}) \leq \frac{1}{1+c^2}$$
    for $c>0$. Hence, for $c = \sqrt{\frac{1-\alpha}{\alpha}}$, the inequality becomes
    \begin{equation}\label{eq:critical_value_bound}
    \mathbb{P}(T^*_{m,n} \geq c \mid  S_{m,n}) \leq \alpha \implies c_{m,n}(\alpha) \leq \sqrt{\frac{1-\alpha}{\alpha}},
    \end{equation}
    which is independent from $m$ and $n$.
    
    Now, let $\theta_N=E[Z_j | H_1]$. Note that because $H(x)=G(x-\delta)$, under $H_1$, a random observation that comes from the shifted distribution $H$ tends to have a larger rank than one that comes from $G$. Consequently, $\eta_N \to \eta>0$, where $\eta_N=\theta_N-\mu_N$. Then, under $H_1$, if $m,n \to \infty$ and $\frac{n}{N}\to \gamma \in (0,1)$,
\begin{align}\label{eq:asymptotic_mean}
    \mathbb{E}\big[T_{m,n} \mid H_1\big] 
    = \frac{n\eta_N}{\sigma_N \sqrt{n\left(1 - \frac{n - 1}{N - 1}\right)}} = \frac{\sqrt{n} \, \eta_N}{\sigma_N \sqrt{1 - \frac{n - 1}{N - 1}}} \to C\sqrt{n}
\end{align}

    where $C=\eta\ /(\sigma_Z\sqrt{(1-\gamma)})$ and $\sigma_Z = \lim_{N\to \infty}\sigma_N>0$. 
    
    Also,
   $$\operatorname{Var}\!\left(\sum_{j=1}^n Z_j \,|\, H_1\right)
    \leq n \operatorname{Var}(Z_1 \mid H_1) \le \frac{n}{16}$$
    where the first inequality comes from sampling without replacement and the second inequality comes from the fact that $Z_1$ is bounded in $[0,1/2]$. Then,
\begin{equation}\label{eq:variance_bound}
    \operatorname{Var}(T_{m,n} \mid H_1) 
    \leq \frac{1}{16\sigma_N^2\left(1 - \frac{n - 1}{N-1}\right)} \leq c
\end{equation}
   for some $c\in \mathbb{R}$ for sufficiently large $n$ and $m$.
   
   Now, consider
\begin{align}\label{eq:type2_bound}
    \mathbb{P}\bigg(T_{m,n} \leq c_{m,n}(\alpha) \mid H_1 \bigg) 
    &= \mathbb{P}\bigg(-T_{m,n} + \mathbb{E}[T_{m,n} \mid H_1] \geq -c_{m,n}(\alpha) + \mathbb{E}[T_{m,n} \mid H_1] \mid H_1 \bigg) \notag \\
    &\leq \mathbb{P}\bigg( \left| T_{m,n} - \mathbb{E}[T_{m,n} \mid H_1] \right| \geq \mathbb{E}[T_{m,n} \mid H_1] - c_{m,n}(\alpha) \mid H_1\bigg) \notag \\
    &\leq \frac{\operatorname{Var}(T_{m,n} \mid H_1)}{\big( \mathbb{E}[T_{m,n} \mid H_1] - c_{m,n}(\alpha) \big)^2}
\end{align}
   where the last inequality comes from Chebyshev's inequality for large enough $n$ s.t. $\mathbb{E}\big[T_{m,n} | H_1\big] >c_{m,n}(\alpha)$. Then, (\ref{eq:critical_value_bound}), (\ref{eq:asymptotic_mean}), (\ref{eq:variance_bound}) and (\ref{eq:type2_bound}) imply that 
   $$\mathbb{P}\left(T_{m,n} \geq c_{m,n}(\alpha) \mid H_1 \right) \to 1, \quad \text{if}\quad m,n\to \infty \quad \text{and} \quad n/N \to \gamma \in (0,1).$$
\end{enumerate}
\end{proof}

\section{Manifold Learning Algorithms NPE and LPP}\label{AppManLearn}
Suppose we observe $m$ $D$-dimensional Phase I observations $Y =\{Y_{1},\cdots, Y_m\}$. To obtain the embedding function $\hat{f}:\mathcal{M}\to \mathbb{R}^d$ using LPP or NPE, we follow the procedures given below.
\subsection{Locality Preserving Projections}
Locality Preserving Projections \citep{lpp} for dimensionality reduction is described as follows,
\begin{enumerate}
\item (Building the adjacency graph) We simply put an edge in between node $i$ and $j$ if $Y_i$ and $Y_j$ are close enough. There are two ways to define this "closeness".
\begin{itemize}
  \item Nodes $i$ and $j$ are connected if $||Y_i-Y_j||^2< \epsilon$ where the norm is Euclidean norm in $\mathbb{R}^D$.
  \item Nodes $i$ and $j$ are connected if $Y_i$ is among k-nearest neighbors of $Y_j$ or vice versa.
\end{itemize}
\item (Setting the weights) Although nodes $i$ and $j$ are connected, it is also possible to measure that closeness by setting weights to the edges. There are two variations to do this.
\begin{itemize}
  \item if node $i$ and $j$ are connected, assign the weight of the associated edge as
  $$W_{ij}=\exp \big \{-\frac{||Y_i-Y_j||^2}{t_0}\big \}$$
  This is called the heat kernel.
  \item The second way to set $W_{ij}=1$, this is equivalent to having $t_0=\infty$ in heat kernel
\end{itemize}
\item Let $K \in \mathbb{R}^{m\times m}$ be a diagonal matrix s.t. $K_{ii}=\sum_{j=1}^mW_{ij}$ and $L=K-W \in \mathbb{R}^{m \times m}$. Then, solve the generalized eigenvalue problem $YLY^{\top}f=\lambda YKY^{\top}f$. That is, find the eigenvectors and eigenvalues such that
    $$YLY^{\top}f_1=\lambda YDY^{\top}f_1$$ $$...$$
    $$YLY^{\top}f_{d}=\lambda YDY^{\top}f_{d}$$
    Notice that all eigenvalues are non-negative.
    $$0\leq \lambda_1 \leq \lambda_2 \leq ... \leq \lambda_{d}$$
    Find embedding of $Y_t$, which is $y_t\in \mathbb{R}^d$, by using the eigenvectors that are associated with the smallest $d $ eigenvalues, that is,
    $$y_t= UY_t$$
    where $U^{\top}=[f_1,f_2,...,f_d]$
\end{enumerate}

\subsection{Neighborhood Preserving Embedding}

Neighborhood Preserving Embedding \citep{npe} for dimensionality reduction is described as follows,

\begin{enumerate}
\item (Building the adjacency graph) This step is the same as the first step of LPP. For the sake of notation, let $N_i$ denote the set of points that are found to be close enough to $Y_i$.
\item (Assigning the weights) This step utilizes the fact that a manifold locally resembles Euclidean space. We optimally construct each point $Y_i$ from the points in its neighborhood by assigning weights. Then we obtain the weighted graph $W\in \mathbb{R}^{m\times m}$ whose elements are denoted by $W_{ij}$
$$W=\arg \min_W \sum_{i=1}^m ||Y_i-\sum_{i=1}^m W_{ij}Y_j||^2$$ $$\text{s.t.}$$ $$\sum_{j=1}^m W_{ij}=1 \quad \forall i\in\{1,2,...,m\}$$ $$W_{ij}=0 \quad \text{if} \quad Y_j\not\in N_i$$
\item (Spectral Embedding) Solve the following generalized
eigenvector problem, $YMY^{\top}f=\lambda YY^{\top}f$, where $M=(I_m-W)^{\top}(I_m-W)$.
Embed $Y_t$ by using the eigenvectors that are associated with the smallest $d$ eigenvalues, that is,
    $$Y_t= Uy_t$$
    where $U^{\top}=[f_1,f_2,...,f_d]$
\end{enumerate}

\section{Parallel Tennessee Eastman Simulator Description}\label{AppTE}
We use the revised TE simulator of \cite{BATHELT2015309}, which includes feedback control loops and consists of 85 observed variables to monitor and 28 different disturbances which can be introduced at varying amplitudes. This TE simulator is shown in Figure \ref{tep}. 

This revised simulator consist of three different MATLAB Simulink models for two different operating modes of the TE Process, together with their feedback control strategies described in \cite{ricker1996decentralized} and \cite{larsson2001self}. 

Although 85 variables are available, not all are relevant to our framework: 43 cannot be sampled at equidistant time intervals, and 12 manipulated variables are not subject to measurement error. We therefore focus on the remaining 30 variables. This leaves 30 usable variables, which by itself does not constitute a high-dimensional setting. To create a more challenging problem and better evaluate our proposed methods, we replicated 10 independent instances of the TE Simulator to construct a 300-dimensional process shown in Figure \ref{fig:combined-te-simulator}. 
\begin{figure}[ht]
    \centering
    \begin{tikzpicture}[
      >=Latex,
      node distance=4mm,
      sim/.style={draw, rounded corners, minimum width=3.4cm, minimum height=7mm, align=center},
      big/.style={draw, thick, rounded corners, minimum width=5.6cm, minimum height=12mm, align=center},
      labelsmall/.style={font=\footnotesize, align=center},
      labelbig/.style={font=\large, align=center}
    ]
    
    \def\nsims{10}
    \def\xgap{6.5}     
    
    \node[sim] (sim1) {TE Simulator 1};
    \foreach \i in {2,...,\nsims}{
      \node[sim, below=of sim\the\numexpr\i-1\relax] (sim\i) {TE Simulator \i};
    }
    
    \draw[decorate,decoration={brace,mirror, amplitude=4pt}]
      ($(sim1.north west)+(-6pt,6pt)$) -- ($(sim\nsims.south west)+(-6pt,-6pt)$)
      node[midway, xshift=-14pt, rotate=90, labelbig] {10 TE Simulators};
    
    \node[big, right=\xgap of $(sim1)!0.5!(sim\nsims)$] (combo) {Replicated TE Simulator\\ \textbf{300 variables}};
    
    \foreach \i in {1,...,\nsims}{
      \draw[->] (sim\i.east) -- (combo.west);
    }
    
    \draw[->, thick, red]
      ($(sim1.west)+(-1.2,0)$) -- (sim1.west)
      node[pos=0, left=2pt, labelsmall, black, anchor=east] {};
    \node[labelsmall, red, anchor=south west] at ($(sim1.north west)+(0,4pt)$)
      {};
    
    \end{tikzpicture}
    \caption{Layout of the replicated TE simulator: Ten TE simulators are combined into a single system with 300 variables, with the red arrow indicating the only location where the faults are injected.}
    \label{fig:combined-te-simulator}
\end{figure}
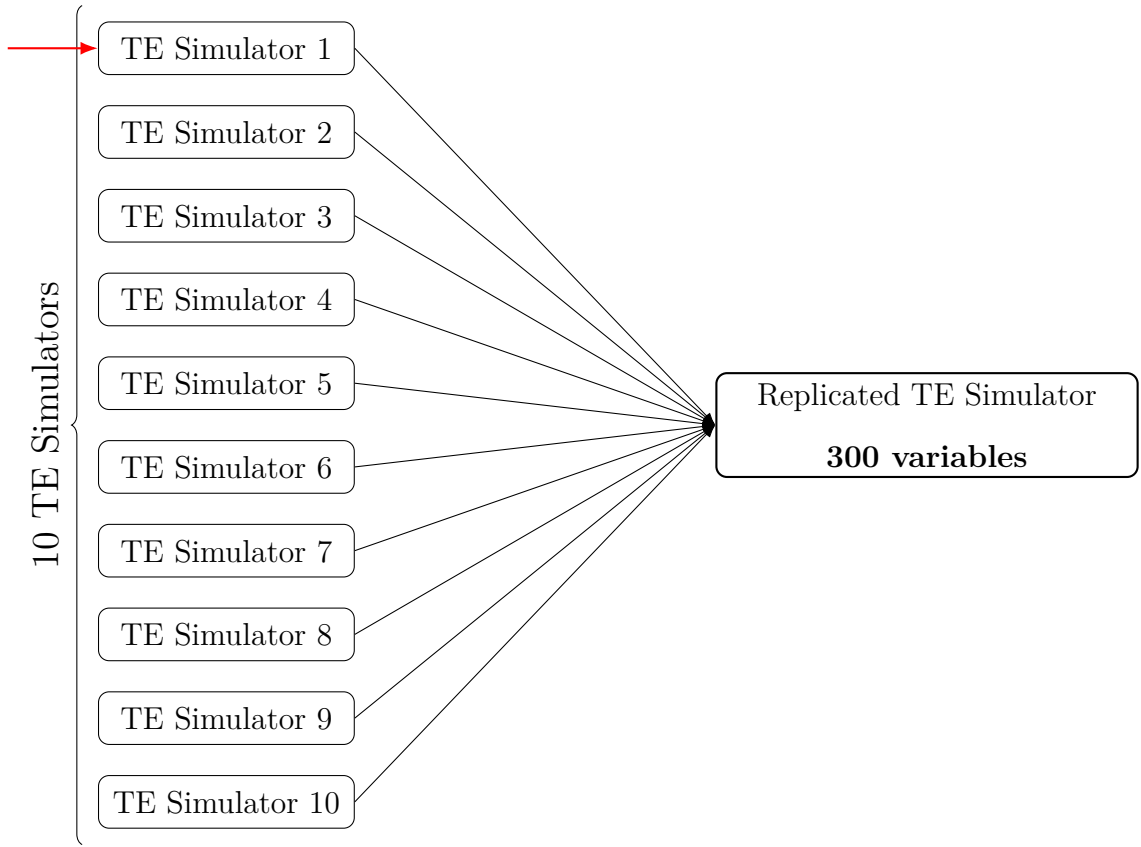

In addition, since the original simulator does not permit fault injection directly from MATLAB scripts, we integrated a Simulink module that allows users to specify the start time, end time, and amplitude of injected faults.

\section{Programs and datasets provided}
\begin{description}
\item[Code:] The code is publicly available at \url{https://github.com/iburaktas/Manifold-SPC}. 
It includes simplified versions of the simulation experiments (Synthetic Process and TE Process) 
and the complete surface anomaly detection experiment presented in the paper. 
In addition, it provides a general pipeline for applying the proposed SPC frameworks to other datasets.

\item[Environment:] The implementation is written in Python, except for the TE Process 
data generation (implemented in MATLAB). A \texttt{requirements.txt} file is provided in the repository 
to install the necessary dependencies.

\item[Models:] The following models are implemented in the \texttt{src/models} folder:
\begin{itemize}
    \item Manifold Fitting Method
    \item LPP (Locality Preserving Projections)
    \item NPP (Neighborhood Preserving Projections)
\end{itemize}

\item[Datasets:] The following datasets and generators are included in the \texttt{src/data} folder:
\begin{itemize}
    \item The Kolektor surface-defect dataset
    \item 100 IC and 60 OC realizations of the TE Process for ARL evaluation with 10 simulations
    \item A data generator for the synthetic process
\end{itemize}

\item[Experiments:] The results of the simplified simulation experiments and the complete 
surface anomaly detection experiment are provided in the \texttt{experiments} folder. Scripts to run these experiments, along with instructions, are available in the \texttt{README.md} file.

\item[Usage on other datasets:] An example of how to use the pipeline with a new dataset 
is described in the \texttt{README.md} file.
\end{description}

\end{document}